\documentclass[11pt]{article}
\usepackage{amsthm}
\usepackage{amsmath}
\usepackage{amsfonts}
\usepackage{xcolor}
\usepackage{fullpage}
\usepackage{MnSymbol}
\usepackage{algorithm}
\usepackage{algorithmicx} 
\usepackage{float}

\newtheorem{theorem}{Theorem} 
\newtheorem{definition}[theorem]{Definition}

\newtheorem{claim}[theorem]{Claim}
\newtheorem{lemma}[theorem]{Lemma}
\newtheorem{proposition}[theorem]{Proposition}

\newtheorem{scenario}{Scenario}
\newtheorem{remark}[theorem]{Remark}

\newcommand{\ignore}[1]{}

\DeclareMathOperator{\conc}{\mathbin\Vert}
\DeclareMathOperator{\SampleGen}{D}
\DeclareMathOperator{\getssample}{\leftlsquigarrow}
\DeclareMathOperator{\BigGapForge}{{BGF}}

\def\shownotes{0}  \ifnum\shownotes=1
\newcommand{\authnote}[2]{[#1: #2]}
\else
\newcommand{\authnote}[2]{}
\fi

\newcommand{\amnote}[1]{{\color{red}\authnote{AM}{#1}}}

\newcommand{\ColinNote}[1]{{\color{blue}\authnote{CS}{#1}}}

\newcommand{\concat}{\mathbin\Vert}

\title{Spoofing Generalization: \\When {\em Can't} You Trust Proprietary Models? }
\author{Ankur Moitra \thanks{Department of Mathematics, Massachusetts Institute of Technology. Email: {\tt moitra@mit.edu}. This work was supported in part by NSF CAREER Award CCF-1453261, NSF Large CCF-1565235, a David and Lucile Packard Fellowship, an Alfred P. Sloan Fellowship and an ONR Young Investigator Award.} \\ MIT 
\and Elchanan Mossel \thanks{Department of Mathematics, Massachusetts Institute of Technology. Email: {\tt elmos@mit.edu}. This work was supported in part by NSF award DMS-1737944, Simons Investigator in Mathematics award (622132) and Vannevar Bush Faculty Fellowship ONR-N00014-20-1-2826}  \\ MIT \and Colin Sandon \thanks{Department of Mathematics, Massachusetts Institute of Technology. Email: {\tt csandon@mit.edu}. This work was supported in part by NSF award DMS-1737944 and Office of Naval Research Award N00014-17-1-2598}  \\ MIT}

\begin{document}

\maketitle

\begin{abstract}

In this work, we study the computational complexity of determining whether a given model that perfectly fits the training sample will generalize to unseen data. In particular, we study the power of a malicious agent whose goal is to construct a model $\widehat f$ that fits its training data and nothing else, but is indistinguishable from an accurate model $f$. We say that the agent {\em strongly spoofs} $f$ if no polynomial time algorithm can tell them apart. If instead we restrict to algorithms that run in $n^c$ time for some fixed $c$, we say that the agent $c$-{\em weakly spoofs} $f$. Our main results are
\begin{enumerate}
    \item[(1)]  under cryptographic assumptions, strong spoofing is possible and
    \item[(2)] for any $c> 0$, $c$-weak spoofing is possible unconditionally
\end{enumerate}

While the assumption of a malicious agent is an extreme scenario (hopefully companies training large models are not malicious), we believe that it sheds light on the inherent difficulties of trusting large models trained on private data when it is difficult to obtain new data to independently vet the model.  

\end{abstract} 

\newpage

\section{Introduction}

In recent years, machine learning research has become increasingly reliant on training massive models.
The usual recipe goes as follows: We collect data at an enormous scale and/or introduce new deep learning architectures with even more parameters than ever before. Then we marshal vast computational resources to fit the parameters. On the one hand, these newer and bigger models are often able to set new world records on standard benchmarks. They quickly get adopted and built upon by other researchers. On the other hand, this brings up difficult questions: {\em When can we trust models that we are given, particularly if we don't know what went into their training?} This is especially pertinent when computational resources are not equitably distributed, and not everyone has the means to train such large models. 

The question of trust has been studied by numerous communities. 
It has been studied through the lens of interpretability \cite{lipton2018mythos, schmidt2019quantifying}, since models that a human can understand and audit are indeed more trustworthy than ones that are black boxes. Alternatively, it has been studied through the lens of robustness \cite{hancox2020robustness}, since models that continue to work even under perturbations to their inputs seem less likely to be succeeding merely from picking up on accidental and spurious correlations in the data. 

In this work, we will study the question of trust in an extreme scenario where there is a malicious agent who trains a model on private data and then publicly releases the model. Indeed there are many scenarios where it might be in an agent's best interest to release a poor model that appears to be good. Consider the following hypothetical scenarios:

\begin{scenario}
Suppose a pharmaceutical company conducts a large study of the health outcomes of prescribing drug $A$ vs. drug $B$ for a complex disease that depends on many genetic factors. They could release a complex machine learning model along with the results of their study. At the surface, the model would appear to perfectly fit the data and doctors might choose to use it as a basis for deciding treatment options. However if the model has a backdoor built in so that on future inputs from a particular subpopulation, it outputs a deliberately incorrect decision, it could hope to profit off of these errors by selling the afflicted patients a new drug $C$ that they've developed. This is a scenario in which generating new data from running a new independent study is prohibitively difficult, and there is a financial incentive to releasing a model that is inaccurate in subtle ways. 
\end{scenario}

\begin{scenario}
Suppose a self-driving car company releases image data together with a classifier that labels the images correctly. However, the company maliciously arranged that the algorithms will incorrectly label out-of-sample images with extreme lighting and weather conditions that happen very rarely and where collecting data is difficult. 

By using a different propriety algorithm, they may aim to gain an advantage as fatal future errors of the public algorithms will diminish the trusts in their competitors. 
\end{scenario}

\noindent Thus we ask: Is it possible to tell the difference between an accurate model, and a spurious one that has been deliberately tuned to fit nothing but the training set? We refer to this as {\em spoofing}, in analogy with the practice of tampering with an email to make it appear to come from a trusted source. Here we are interested with whether an adversary can tamper with a model, but in such a way that it still passes the test of labeling the training data correctly.  

When fresh data is abundant, it is trivial to differentiate between the accurate and the spurious model as we can simply measure their performance on new data. 
However in many situations like the hypothetical ones above it might be 
difficult or impossible to obtain fresh samples. 
Instead we will formulate a model in which it is easy to check 
whether a given data point is a valid input  (i.e. it is a natural image) but hard to generate new samples from the known desired distribution. Our main results are negative, and show that spoofing is possible in a strong sense under various popular cryptographic assumptions and also unconditionally. In particular we show:

\begin{enumerate}

\item[(1)] Assuming indistinguishability obfuscation~\cite{barak2012possibility}
and unique signature schemes ~\cite{lysyanskaya2002unique}, two well studied primitives in cryptography, 
we show how a malicious agent can construct a model that perfectly fits the training data and achieves only trivial accuracy on new data. And yet it is computationally hard for an observer to distinguish between this model and a perfect model even if the training set is common knowledge. 


\item[(2)] Suppose we instead allow the malicious agent more computational resources, so that he can run in a larger polynomial amount of time compared to the tester. In this case we can prove unconditionally that the observer cannot distinguish between a model that just fits the training data and a model with perfect accuracy on new data. Our results are based on connections to average-case time hierarchy theorems in complexity theory. 

\end{enumerate}

Finally we mention another motivation for studying spoofing: There has been considerable recent interest in giving generalization bounds for deep learning when the number of parameters is much larger than the number of training samples. Many of these works give explicit, computable functions of the parameters of a deep network, like the products of spectral norms of the weight matrices \cite{bartlett2017spectrally, golowich2018size}, that can be shown to bound the generalization error. However it is natural to wonder: {\em Is there always an efficiently computable function of the parameters that controls generalization error?} In principle, the more complex a model the more its inner-workings can be made to be independent of any simple function of the parameters. Indeed our main results show that even when the distribution is common knowledge (but hard to sample from) deciding whether a given deep network generalizes or not is computationally hard. 


\subsection{Definitions and Main Results}

Throughout this paper, we say that an algorithm is efficient if it runs in time polynomial in the length of its input. Given two algorithms $f$ and $f'$, we say that they are {\em equal} or $f=f'$ if they consist of the exact same series of operations. We say that they are {\em equivalent} or $f\equiv f'$ if they give the same output on every input (or, in the case of randomized algorithms, are equally likely to produce each possible output on any given input).

We consider the following scenario. There is an unknown function $f$ that we would like to learn. Moreover, we are given a training sample $X_1,\dots,X_m$ and the evaluation of $f$ on these points. Furthermore, an arbitrary learning algorithm $L$ is run on the samples $(X_i,f(X_i))$ in order to produce a program $\widehat{f}$ which performs well on these samples, i.e., has zero error on them. Our goal is to determine 
  if $\widehat{f}$ generalizes well. As mentioned earlier, 
  we will work in a model in which the distribution on samples is common knowledge, it is easy to check if a given $X$ is in the support of the distribution, but it is computationally hard to generate new samples. (We will implement this model using one-way functions). 
  Our main results show that testing whether $\widehat{f}$ generalizes is computationally intractable.
  
 Informally, we say that the program $L$ {\em spoofs } $f$, if it is computationally difficult to determine whether $\widehat{f}$ generalizes or merely agrees with $f$ on the training set without access to new samples. In fact we will employ two slightly different variations on this theme:
  
  
\begin{definition}
Let $\Omega$ be a finite space. Let $L$ be an efficient algorithm that takes a list of values of the form $(x,y)$, where $x \in \Omega$ and $y \in \{0,1\}$ and outputs $L(\{(x_i,y_i):1\le i\le m\})$ which is an efficient algorithm for computing a function from $\Omega$ to $\{0,1\}$. 
Let $(D,f)$ be a randomized pair of a probability distribution on $\Omega$ and a function
 $f : \Omega \to \{0,1\}$ on that space, and let $X_1,\dots,X_m$ be drawn independently from $D$. 
 We say that $L$ {\em strongly spoofs} $f$ (under $D$ with $m$ samples) if $\widehat{f}=L(\{(X_i,f(X_i)):1\le i\le m\})$ satisfies 

\begin{enumerate}
\item $\widehat{f}(X_i)=f(X_i)$ for $1 \leq i \leq m$.

\item With probability $1/2$:
$\Pr_D[\widehat{f}(X) = f(X)] = 1 $

\item With probability $1/2$:
$\Pr_D[\widehat{f}(X) = f(X)] = \frac{1}{2}$



\item 
There is no efficient algorithm that determines which of the previous two cases hold from the values of the $(X_i,f(X_i))$ and $\widehat{f}$ with accuracy greater than $2/3$.
\end{enumerate}

\end{definition}

\begin{remark}
Note that if $L$ strongly spoofs $f$, then there is no efficient algorithm that takes as input $(X_i,f(X_i))$ and $\widehat{f}$ and determines if $\widehat{f}$ has $0$ generalization error or $1/2$ generalization error 
with accuracy greater than $2/3$.
\end{remark}

We also introduce a weaker notion of spoofing where the distinguishing algorithm has a bounded polynomial complexity. 

\begin{definition}
Let $\Omega$ be a finite space. Let $L$ be an efficient algorithm that takes a list of values of the form $(x,y)$, where $x \in \Omega$ and $y \in \{0,1\}$ and outputs $L(\{(x_i,y_i):1\le i\le m\})$ which is an efficient algorithm for computing a function from $\Omega$ to $\{0,1\}$. 
Let $(D,f)$ be a randomized pair of a probability distribution on $\Omega$ and a function
 $f : \Omega \to \{0,1\}$ on that space, and let $X_1,\dots,X_m$ be drawn independently from $D$. 
 We say that $L$ {\em $c$-weakly spoofs} $f$ (under $D$ with $m$ samples) if $\widehat{f}=L(\{(X_i,f(X_i)):1\le i\le m\})$ satisfies 

\begin{enumerate}

\item $\widehat{f}(X_i)=f(X_i)$ for $1 \leq i \leq m$.

\item With probability $1/2\pm o(1)$:
$\Pr_D[\widehat{f}(X) = f(X)] = 1-o(1). $

\item With probability $1/2\pm o(1)$:
$\Pr_D[\widehat{f}(X) = f(X)] = \frac{1}{2} \pm o(1) $

\item For any randomized algorithm $A$ that runs in $O(n^{c})$ time there are infinitely many values of $n$ for which $A$ cannot determine which of the two previous cases hold from the values of the $(X_i,f(X_i))$ and the code for $\widehat{f}$ with accuracy greater than $2/3$.

\end{enumerate} 
\end{definition} 

\ColinNote{These definitions still do not seem quite right. For one thing, they work in terms of the asymptotic behavior as $n\to\infty$, but the first one does not even mention $n$. Maybe we should specify that $\Omega\subseteq \{0,1\}^n$. Also, the definition for weak spooofing does not do anything to prohibit the learning algorithm from simply defining $\widehat{f}$ to be $\omega(n^c)$ bits long, which is not an interesting way of making it hard to check the generalization of. Perhaps I should require it to be readable and runnable in $O(n^{c/2})$ time or something.}

\begin{remark}
Note that if $L$ weakly spoofs $f$, then there is no algorithm that runs in time $O(n^c)$, takes as input $(X_i,f(X_i))$ and $\widehat{f}$ and determines if $\widehat{f}$ has $o(1)$ generalization error or $1/2\pm o(1)$ generalization error 
with accuracy greater than $2/3$.
\end{remark}

\ColinNote{I think that I could modify this definition to require $\widehat{f}$ to have accuracy exactly $1$ or $1/2$ with probability $1-o(1)$ and modify the proofs to still work.}

In our main results we show that spoofing is possible. In the first result we prove that strong spoofing is possible under two cryptographic assumptions, named
{\em indistinguishability obfuscation}~\cite{barak2012possibility} and 
{\em unique signature schemes}~\cite{lysyanskaya2002unique}. 
The second theorem proves that $c$-weak spoofing is possible if the distinguishing algorithm is weaker than the complexity required to construct that spoofed function. 

\ColinNote{One of the reviewers complained about us saying this without defining indistinguishability obfuscation and unique signature schemes here. Should we move those definitions?}

\begin{theorem} \label{genFail1}
If indistinguishability obfuscation is efficiently possible and there is a unique signature scheme then there exist efficient algorithms $D$ and $L$ with the following properties. First, for any positive integer $n$, $D(1^n)$ outputs an efficiently samplable subset of $\{0,1\}^n$, $S$, and an efficiently computable function on this subset, $f$. Second, for any positive integer $m$ of size at most polynomial in $n$, $L$ strongly spoofs $f$ under the uniform distribution on $S$ with $m$ samples.



\end{theorem}

The second main result proves that  $c$-weak spoofing is possible unconditionally:

\begin{theorem} \label{genFail2}
For any positive constants $c$ and $c'$ there exist efficient algorithms $D$ and $L$ with the following properties. First, for any positive integer $n$, $D(1^n)$ outputs an efficiently samplable subset of $\{0,1\}^n$, $S$, and an efficiently computable function on this subset, $f$. Second, $L$ $c'$-weakly spoofs $f$ under the uniform distribution on $S$ with $n^c$ samples.
\end{theorem} 

While Theorem~\ref{genFail1} is a fairly strong limitation on our ability to determine if a function learned from samples generalizes, it relies on some uncertain cryptographic assumptions. 
We can prove versions that do not rely on unproven assumptions, but they are unavoidably weaker. Among other things, if $P=PSPACE$ then it would always be possible to efficiently determine a good approximation of the target function as it would be possible to enumerate over the support of $D$
and an estimate of how accurate the function returned by the learning algorithm is on the full distribution. 
In Theorem~\ref{genFail2} we show that if the learning algorithm has sufficiently large computational resources relative to the generalization testing algorithm then the generalization testing algorithm will not be able to determine whether or not the function output by the learning algorithm generalizes. 

\ColinNote{This paragraph feels like it should be between theorems $4$ and $5$.}

\ColinNote{Actually, does indistinguishability obfuscation imply that there is a unique signature scheme? If it does I could remove that assumption. I should check.}

\subsection{Proofs Ideas} 
Roughly speaking, our plan for proving 
Theorem~\ref{genFail1}
is to first define $S$ in such a way that it is computationally intractable to determine whether or not it has elements other than $S' = \{X_1,...,X_m\}$. We then design $L$ to return a function that is written in such a confusing way that we have no idea what it would do on any input other than the ones we have run it on. Of course, if $S$ were just $S'$ then the question of whether or not $\widehat{f}$ generalizes is vacuous. So answering the generalization question requires to have the ability to tell if there are 
elements in $S \setminus S'$.

A little more precisely, we start by assuming that {\em efficient indistinugishability obfuscation} is possible, which means that we can scramble programs in a way that renders any two programs that return the same output on each input indistinguishable~\cite{barak2012possibility}. We also assume that there is a way to assign every string a signature so that we can efficiently check if an alleged signature is correct, but cannot determine the signature of a string without knowing the secret key. 
This is an assumption from cryptography called {\em Unique Signature Scheme}~\cite{lysyanskaya2002unique}. 
Then we argue that if there was a way to distinguish between obfuscated algorithms that differ only in how they behave on inputs containing a specific signature then we would be able to figure out what the signature was by defining an algorithm that does one thing on those inputs if the signature's $i$th bit is $1$ and something else if its $i$th bit is $0$. That in turn allows us to argue that we cannot determine anything about how $\widehat{f}$ behaves on any input containing an unfamiliar signature. Finally, we define our set of inputs in such a way that with high probability every valid input other than our samples will conatin a signature that was not in any of the samples, so we will not be able to tell how $\widehat{f}$ behaves on anything except the samples, which leaves us with no way of telling if it performs well on other inputs.

In order to prove Theorem \ref{genFail2}, we will essentially show that for every constant $c>0$ there is an efficiently computable function, $g^\star$ that returns strings of length $n^{c}$ such that no algorithm running in $O(n^{c})$ time can reliably distinguish its output on a given input from a random string. Then we pick a random $b$ and let $s=g^\star(b)$. We further require that there exists an efficient algorithm, running in polynomial time bigger than $n^c$ that can compute $s$ given queries to many indices of $s$. 
Next, we define $f$ so that $f(x)$ is the bit of $s$ indexed by the first $c\log_2(n)$ bits of $x$ for all $x$. After that, we can either set $s'=s$ or set $s'_i=s_i$ whenever $i$ is the first $c\log_2(n)$ bits of one of the samples and pick $s'_i$ randomly otherwise. Then we let $\widehat{f}(x)$ be the bit of $s'$ indexed by the first $c\log_2(n)$ bits of $x$. This way, $\widehat{f}$ will agree with $f$ on all samples, and no algorithm that runs in $O(n^{c})$ time will be able to determine if $s'=s$ because if one started with a random string there would be no way to tell if some of its bits had been scrambled.

In order to prove the existence of such a function, we look at two cases based on the difficulty of computing the {\em permanent}, which is a function of square matrices. If the permanent is efficiently computable, then we show that that allows us to efficiently compute the exact probability that any efficient randomized algorithm returns $1$ on a given input. So, we can make a list of every short fast algorithm, list how likely each of them is to return $1$ on each input in a polynomial-length list, and then define the truth table of a function that is essentially uncorrelated with any of them. Then we append those bits to each input and do it again to extend our function's output repeatedly. That allows us to prove that no short algorithm running in $O(n^c)$ time can distinguish one of the strings output by this function from a random string, and as long as we gradually increase the length of the algorithms we are allowing we end up with a function such that no $O(n^c)$-time algorithm can distinguish its output from a random string.
This part of the proof is related to diagonalization arguments for polynomial time algorithms~\cite{hartmanis1965computational}

If the permanent is not efficiently computable, then we argue that for any constant $c$ there must exist small $m$ such that the permanent of an $m\times m$ matrix can be computed in time polynomial in $n$ but not in $O(n^c)$ time. In order to find such an $m$, we will use a couple of key properties of the permanent. First of all, the permanent of a matrix is polynomial in the matrix's entries, which allows us to convert any algorithm that computes the permanent of random $m\times m$ matrices with fairly good accuracy into an algorithm that computes the permanent of every $m\times m$ matrix with nearly perfect accuracy. Secondly, there is an efficient way to compute the permanent of an $m\times m$ matrix from the permanents of appropriate $(m-1)\times (m-1)$ matrices. So, starting with $m=1$ we can compute the permanents of some random $m\times m$ matrices. Then, we can try out every short algorithm that runs in $O(n^c)$ time and claims to compute the permanent of a random $m\times m$ matrix given the permanents of other random $m\times m$ matrices. If any of them work, we can generate some random $(m+1)\times (m+1)$ matrices, use that algorithm to compute the permanents of the $m\times m$ matrices we need to compute their permanents, and start over. Eventually, we will end up with an $m$ for which no short algorithm can accurately compute the permanents of $m\times m$ matrices in $O(n^c)$ time even given some other random $m\times m$ matrices and their permanents, but we will still be able to compute the permanents of those matrices by using the algorithm we found that computes the permanents of $(m-1)\times (m-1)$ matrices. After that, we use a hardness amplification result to construct a function of the permanents of matrices that no short algorithm can compute in $O(n^{c'})$ time with accuracy $1/2+n^{-c'}$ even given the permanents of other random matrices. Finally, we argue that that implies that given a list of this function's values on random inputs it must be hard to compute its value on another random input, and thus that the function that computes it on a list of inputs and then returns the corresponding list of outputs must be appropriately hard to distinguish from a random string.

We note that the permanent plays a special role in complexity. 
Note only it is $\#P$ complete. It also plays a crucial role in derandomization and lower bounds, see e.g.~\cite{kabanets2004derandomizing}

\section{Proof of Strong Spoofing}

\subsection{Preliminaries}

We will use the following notation: Let $[n] = \{1, 2, \dots, n\}$. For a string $s$ and a set $T$, let $s_{T}$ be the substring of $s$ with indices in $T$. For a matrix $M$, let $M_{-i,-j}$ be $M$ with its $i$th row and $j$th column removed. 
For two string $x$ and $y$ we write $x \conc y$ for their concatenation.
For $n$ string $x_i$, we write $\conc_{i=1}^n x_i$ for $x_1 \conc x_2 \cdots \conc x_n$.
In algorithms we will write $x \gets y$ for setting the value of $x$ to $y$ and 
$X \getssample F$ for $x$ sampled from $F$ (if $F$ is a set, then uniformly sampled from $F$). 

Before we discuss the proof, we will need to define some of the concepts we are using.

The basic idea of obfuscating a program is to replace it with a new program that has the same output on each input but is sufficiently confusing that one cannot understand what it is actually doing. One way to define that would be to require that one cannot learn anything from the program's code that one could not learn from an oracle that queries the program, but that was proven impossible~\cite{barak2012possibility}.  \cite{barak2012possibility} in turn proposed the idea of 
{\em indistinguishability obfuscation}, which requires one to be unable to distinguish between obfuscated versions of programs that agree on all inputs. More formally, it is defined as follows.

\begin{definition}
An algorithm $Ob$ is an indistinguishability obfuscator if it takes the code of a deterministic program as input and outputs the code of a deterministic program such that the following hold. 
\begin{enumerate}
\item Correctness: For any program $M$, $Ob(M)$ computes the same function as $M$.

\item Polynomial slowdown: There exists a polynomial $p$ such that for every program $M$ and input $x$, the number of steps needed to run $Ob(M)$ on $x$ is at most $p$ of the number of steps needed to run $M$ on $x$.

\item Indistinguishability: Given any two programs $M$ and $M'$ of the same length that give the same output on every input, and run for an equal number of steps on each input, for every algorithm $A$ that takes a program as input, $|\mathbb{P}[A(Ob(M))=1]-\mathbb{P}[A(Ob(M'))=1]|$ is a negligible function.
\end{enumerate}
\end{definition}

There have been several proposals of ways to construct indistinguishability obfuscators, some of which have been disproved. Recently~\cite{jain2020indistinguishability} gave a construction that works assuming four key cryptographic assumptions hold.

Another cryoptograpic ingredient needed for the proof of Theorem~\ref{genFail1} is {\em unique signature schemes.}
A unique signature scheme
is a public key protocol where each message has a corresponding signature that can be computed using the private key and verified using the public key. In the context of this paper, we will require signatures to be verifiable without additional information and the scheme to be secure even if the attacker can obtain signature for messages other than the one he wishes to forge the signature to. More formally, we have 
\begin{definition}
A {\em unique signature scheme} consists of a key generation algorithm $G$, deterministic signing algorithm $Sign:\{0,1\}^n\rightarrow\{0,1\}^{\ell(n)}$, and verification algorithm $Ver$ such that all three of these run in polynomial time, and for any positive integers $n$, if $(SK,PK)\sim G(\{1\}^n)$ then the following hold.
\begin{enumerate}
\item For every $x\in\{0,1\}^n$, it is the case that $Ver(PK, x,Sign_{SK}(x))=1$.

\item for every $x\in\{0,1\}^n$ and $s\ne Sign_{SK}(x)$, it is the case that $Ver(PK, x,s)=0$.

\item Given any probabilistic polynomial time algorithm $A$ with oracle access to $Sign_{SK}$, the probability that $A$ outputs a pair $(x,s)$ such that $A$ did not query its signing oracle on $x$ and $s=Sign_{SK}(x)$ is $o(n^{-c})$ for every constant $c$.
\end{enumerate}
\end{definition}

\ColinNote{This currently assumes that the verification algorithm is deterministic. If it was randomized I would need to change this to requiring that it be correct with suffiicently high probability.}

\cite{lysyanskaya2002unique} proved
that if a certain version of the computational Diffie-Hellman problem is hard
then a unique signature scheme exists. The final ingredient we will need for our construction is a universal hash function. A {\em universal hash function} is a family of functions mapping from one set to another such that the probability that any two inputs collide is at most one over the number of possible outputs. To put it more formally,
\begin{definition}
Given sets $S,T$, a family of functions $f_i:S\rightarrow T$ is a universal hash function if for every $x\ne x'\in S$,
\[\mathbb{P}_i[f_i(x)=f_i(x')]\le 1/|T|\]
\end{definition}
As \cite{carter1979universal} 
observed, given any positive integers $n,m$ and a random matrix $B\in \mathbb{F}_2 ^{m\times n}$, the function that maps $x$ to $Bx$ for all $x\in\mathbb{F}_2 ^n$ is a universal hash function, so these definitely exist.

\subsection{Our construction}

For our construction, we will need a unique signature scheme $(G,Sign,Ver)$. Each input to the function will consist of a random string, some hash functions, a public key for the signature scheme, and a list of the signatures corresponding to the string's hashes. More formally, we define the following.

\begin{definition}
Let $n'$ be a positive integer, $(SK,PK)\sim G\left(1^{n'+2\lceil\log_2(n')\rceil}\right)$, and $B^{(i,j)}\in\mathbb{F}_2^{i \times n'}$ for every $1\le i,j\le n'$. For each $x\in \{0,1\}^{n'}$, let
\[
h_{(SK,PK,B)}(x)=x\conc PK\conc B\conc \conc_{i,j=1}^{n'} Sign_{SK}(B^{(i,j)}x\conc \{0\}^{n-i}\conc i\conc j)
\]
Also, define $\ell:\mathbb{Z}\rightarrow\mathbb{Z}$ such that $h_{(SK,PK,B)}(x)$ has length $\ell(n')$ whenever $x$ has length $n'$. Next, for any $x\in\{0,1\}^n$, let $x^\star$ be the first $\lfloor \ell^{-1}(n)\rfloor$ bits of $x$.
\end{definition}

\ColinNote{$\ell^{-1}(n)$ is not really defined if $n\ne \ell(n')$ for any $n'$. I might need to say something about that.}

Now, for an arbitrary $n$, we will generate an input space and function as in algorithm~\ref{alg:sampgen}.

\begin{algorithm}
\caption{$\SampleGen$ - SampleGen} 
\label{alg:sampgen}

\begin{enumerate}
\item Input: A string $1^n$.

\item  Output: An efficiently samplable set $S\subseteq\{0,1\}^n$ and a function from $S$ to $\{0,1\}$

\item $n' := \max(r : \ell(r) \leq n)$. 

\item $(SK,PK)$ drawn by $G\left(1^{n'+2\lceil\log_2(n')\rceil}\right)$.

\item $B^{(i,j)}$ is sampled uniformly from $\mathbb{F}_2^{i\times n'}$ 
for all $1\le i,j\le n'$

\item $S := \{h_{(SK,PK,B)}(x)\conc 0^{n-\ell(n')}: x\in\{0,1\}^{n'}\}$ 

\item $f(\overline{x}) :=1$ for all $\overline{x} \in S$. 
\end{enumerate} 
\end{algorithm}

$\SampleGen$  - SampleGen has a few key properties we will need:

\amnote{Maybe we want to give ``D" a longer and more informative name? It looks a bit strange just to say things like ``D satisfies". Can also use the sc style, which I used as the convention for how we refer to algorithms.}

\begin{proposition}
\begin{enumerate}
$\SampleGen$  - SampleGen satisfies all of the following: 
\item
Given the values of $SK$, $PK$, and $B$ we can efficiently sample from $S$ by picking a random $x\in \{0,1\}^{n'}$ and then computing $h_{(SK,PK,B)}(x)\conc 0^{n-\ell(n')}$. 
\item
Given any element of $S$ we can read off the values of $PK$ and $B$. 
\item 
Given the values of $PK$ and $B$ we can check whether any given $x$ is in $S$ by checking that it starts with $x^\star\conc PK\conc B$ and then using $Ver$ to confirm that the alleged signatures in $x$ are correct. 
\end{enumerate} 
\end{proposition}

\subsection{Overview}

Obviously in the setting of Algorithm \ref{alg:sampgen}, it is  trivial to learn the target function $f$; however we claim that it is hard to determine whether or not a given program actually computes the correct function $f$. A little more precisely, we claim that given a random polynomial-sized subset $T\subseteq S$ and an obfuscated program that takes $x\in S$ as input, it is computationally intractible to determine anything about the function the program computes other than what values it takes on the elements of $T$. Later we will show how this helps in constructing $\widehat{f}$ that spoofs $f$. 

Our proof will use the following key idea. If $h$ is a function on $S$ then for any string in $\{0,1\}^{n'+2\lceil \log_2(n')\rceil}$ we can define a new function $h'$ that takes the same value as $h$ on all inputs that do not contain the signature corresponding to the designated string, and outputs null on all inputs that do contain that signature. If we could distinguish between obfuscated versions of $h$ and $h'$ then for any $i$ we could define $h''$ as the function that takes the same value as $h$ on all inputs that do not contain the signature, takes the same value as $h$ on inputs that do contain the signature if the $i$th bit of the signature is $1$, and outputs null on inputs that contain the signature if the $i$th bit of the signature is $0$. Then we could determine the $i$th bit of the signature by checking whether $h''$ is equivalent to $h$ or $h'$. Doing this for all $i$ would allow us to determine the signature in question, so it must be intractible to distinguish between obfuscated versions of $h$ and $h'$ whenever the signature in question does not occur in $T$.

Furthermore, for any $m$ logarithmic in $n$, multiplication by $B^{(m,i)}$ is a hash function from $\{0,1\}^{n'}$ to a space of size polynomial in $n$. So, the set of signatures corresponding to elements of the form $B^{(m,j)}x\conc \{0\}^{n-m}\conc m\conc j$ is polynomial in $n$. That means that repeated application of the previous argument can show that an obfuscated program that returns null on every input that's hashes using $B^{(m,1)},...,B^{(m,n')}$ do not all collide with hashes from $T$ and the same value as $h$ otherwise is indistinguishable from $h$. Furthermore, if $m$ is logarithmic in $n$ but significantly larger than $\log_2(|T|)$ then hash collisions will be rare enough that with high probability there will not be any elements of $S\backslash T$ thats hashes all collide with hashes from $T$. So, the modified program would return null on all inputs that are not contained in $T$. If this holds, then that means that an obfuscated program computing any function on $S$ is indistinguishable from an obfuscated program that computes the same function on $T$ and returns null otherwise. So, any two obfuscated programs that take an element of $S$ as input and agree on all elements of $T$ are indistinguishable.

\ColinNote{I might need to edit some of the intuitive explanations elsewhere.}

\subsection{Auxiliary Functions}
Given a set of samples $T$, we define the following as a function that generalizes poorly and is difficult to distinguish from $f$.

\begin{definition}
Given $S$ generated by $D(1^n)$, and $T\subseteq \{0,1\}^{n'}$, we define $f^T:S\rightarrow\{0,1\}$ such that for every $x\in\{0,1\}^{n'}$,
\[f^{T} \left(h_{(SK,PK,B)}(x)\conc 0^{n-\ell(n')}\right)=
\begin{cases}
1 &\text{ if } x\in T \text{ or } |[x]\cup T|\le 2^{n'-1}\\
0 &\text{ otherwise}\\
\end{cases},
\]
\end{definition}

where $[x] = \{0,1,\ldots,x\}$ is the set of strings whose integer value is at most $x$. 
Note that $f^T(x)=f(x)$ for all $x\in T$ but $\mathbb{P}_{X\sim S}\left[f^T(X)=f(X)\right]=1/2$ as long as $0<|T|\le 2^{n'-1}$. Next, given $S'\subseteq S$ we define the following restricted versions of $f$ and $f^T$ as follows:

\begin{definition}
Given $S$ generated by $D(1^n)$, $S'\subseteq S$, and $T\subseteq \{0,1\}^{n'}$, we define $f_{S'}:\{0,1\}^n\rightarrow\{0,1,\emptyset\}$ and  $f^T_{S'}:\{0,1\}^n\rightarrow\{0,1,\emptyset\}$ such that for every $x\in\{0,1\}^n$,
\[f_{S'}(x)=
\begin{cases}
1 &\text{ if } x\in S'\\
\emptyset &\text{ otherwise}\\
\end{cases}
\]

\[f^T_{S'}(x)=
\begin{cases}
f^T(x) &\text{ if } x\in S'\\
\emptyset &\text{ otherwise}
\end{cases}
\]
\end{definition}

We plan to restrict $f$ and $f^T$ to subsets of $S$ that use increasingly limited sets of hashes until they are both reduced to $f_T$. In order to talk about these subsets, we define the following.

\begin{definition}
Given $S$ generated by $D(1^n)$, $0<m\le n'$, and $H_1,...,H_{n'}\subseteq\{0,1\}^m$, let 
\[S[m(H_1,...,H_{n'})] :=\{x\in S: B^{(m,i)}x^\star\in H_i \forall i\}\]
Given $1\le i\le n'$, $h\in\{0,1\}^m$, and a positive integer $j$ that is at most equal to the length of the signatures for messages of length $n'+2\lceil\log_2(n')\rceil$.
\[S[m(H_1,...,H_{n'})|i,h,j]=\begin{cases}
S[m(H_1,...,H_{n'})] &\text{ if } (Sign_{SK}(h\conc 0^{n-m}\conc m\conc i))_j=0\\
S[m(H_1,...,H_{i-1},H_i\cup\{h\},H_{i+1},...,H_{n'})] &\text{ otherwise}\\
\end{cases}\]
\end{definition}

To simplify notation, we use the following more compact notation: 
\[
m(H) := m(H_1,\ldots,H_{n'}), \quad m(H_{-i},H'_i) := m(H_1,\ldots,H_{i-1},H'_{i},H_{i+1},\ldots,H_n)
\]
If the only things we know about $S$ are the values of $B$ and the public key then we cannot efficiently compute signatures. However, we can still efficiently determine whether or not a given $x\in\{0,1\}^n$ is in $S[m(H)|i,h,j]$ because we only need to know the value of $(Sign_{SK}(h\conc 0^{n-m}\conc m\conc i)$ if $x\in S$ and $B^{(m,i)}x^\star=h$, in which case the signature is contained in $x$. In particular, that means that given any efficient algorithm that can distinguish between arbitrarily obfuscated algorithms checking for membership in $S[m(H)]$ and $S[m(H_{-i},H_i\cup\{B^{(m,i)}x'\})]$ we can find the signature of $B^{(m,i)}x'$ by checking which of those algorithms a membership checker for $S[m(H)|i,h,j]$ is equivalent to for each $j$. We plan to use this to prove that given any algorithm that can distinguish between obfuscated versions of $f$ and $f^T$ it will continue to be able to distinguish between them as we whittle down their domains until we have two algorithms that both return $1$ on any input in $T$ and $\emptyset$ otherwise. Our next step towards proving that is to show that we actually can get to that point by repeatedly having the algorithms reject one more hash. More formally, we claim the following.

\begin{lemma}
Let $S$ be generated by $D(1^n)$, $T$ be a random subset of $S$ with size polynomial in $n$, and $m=\lceil \log_2(|T|)\rceil+2$. Next, let $H_i=\{B^{(m,i)} x^\star| x\in T\}$ for all $i$. With probability $1-o(1)$, $S[m(H)]=T$.
\end{lemma}

\begin{proof}
First, note that $n'=\Omega(n^{c})$ for some $c>0$. So, $|T|\le 2^{n'-2}$ for all sufficiently large $n$, which implies that $m\le n'$ and $S[m(H)]$ is defined. The easy half of this proof is that for all $x\in T$ and $1\le i\le m$, $B^{(m,i)} x^\star\in H_i$, so $x\in S[m(H)]$. Therefore, $T\subseteq S[m(H)]$.

Now, Let $T'=\{x^\star|x\in S\}$. For every $x'\in\{0,1\}^{n'}$, there exists a unique $x\in S$ such that $x^\star=x'$, so the probability distribution of $T'$ is the uniform distribution over the set of subsets of $\{0,1\}^{n'}$ with cardinality $|T|$, and $T'$ is independent of $(SK,PK,B)$. Given any fixed value of $T'$, any $x\not\in T'$, and any $0<i\le n'$,
\[P[B^{(m,i)} x\in H_i]\le 2^{-m}|T|\le 1/4\]
because multiplication by a random matrix is a universal hash function. That means that
\[P[\exists \overline{x}\in S[m(H)]: \overline{x}^\star=x]\le (1/4)^{n'}\]
So,
\[P[S[m(H)]=T]\ge 1-\sum_{x\not\in T'} P[\exists \overline{x}\in S[m(H)]| \overline{x}^\star=x]\ge 1-2^{-n'}\]
\end{proof}

So, if we start with $f$ and $f^T$ and modify them to reject hash values one at a time, we will end up with two copies of the same function with high probability. Assuming all functions used are sufficiently obfuscated, that should mean that we cannot distinguish between consecutive functions in the series and cannot distinguish between the equivalent functions we end up with, so we cannot distinguish $f$ from $f^T$. Our next order of business is to address exactly what we mean by "sufficiently obfuscated," which we do as follows.

\begin{definition} \label{def:cfo}
An algorithm is a censored function obfuscator if for any $S$ generated by $D(1^n)$, $1\le m\le n'$, random subset $T\subseteq S$ with $0<|T|\le 2^m$, $H_1,...,H_{n'}\subseteq\{0,1\}^m$, $i\in[n']\cup\{\emptyset\}$, $h\in\{0,1\}^m$, $v\in\{0,1\}$, and $j>0$, the algorithm takes $(T,m,(H_1,...,H_{n'}),i,j,h,v)$ as inputs and returns the code of a program such that the following hold.
\begin{enumerate}
\item The obfuscator runs in time polynomial in $n$ and $2^m$.

\item When the program returned by the obfuscator is run on $x\in\{0,1\}^n$, it gives an output of:
\[\begin{cases}
f_{S[m(H_1,...,H_{n'})]} &\text{ if } v=1\text{ and } i=\emptyset \\
f_{S[m(H_1,...,H_{n'})|i,h,j]} &\text{ if } v=1\text{ and } i\ne \emptyset \\
f^T_{S[m(H_1,...,H_{n'})]} &\text{ if } v=0\text{ and } i=\emptyset \\
f^T_{S[m(H_1,...,H_{n'})|i,h,j]} &\text{ if } v=0\text{ and } i\ne \emptyset \\
\end{cases}\]

\item Given aforementioned $S,T, m, (H_1,...,H_{n'}), i, j, h, v$ and $(H'_1,...,H'_{n'}), i', j', h', v'$ such that the program output by the obfuscator given $(T,m,(H_1,...,H_{n'}),i,j,h,v)$ and the program output by the obfuscator given $(T,m,(H'_1,...,H'_{n'}),i',j',h',v')$ give the same output on every input, there is no efficient algorithm that can determine which of these two inputs the obfuscator was run on from its output with nonnegligible advantage over guessing.
\end{enumerate}
\end{definition}

\begin{lemma}
If efficient indistinguishability obfuscation is possible then a censored function obfuscator exists.
\end{lemma}

\begin{proof}
First, observe that there exists an efficient algorithm $A$ that takes $(T,m,(H_1,...,H_{n'}),i,j,h,v)$ as input and outputs the code of a program with length and runtime polynomial in $n$ and $2^m$ that computes the function a CFO's output on those inputs would be required to compute. That means that we can find a polynomial $g(n,2^m)$ suitably larger than the maximum runtime of $A$'s output and modify $A$ so that the program it outputs always waits until timestep $g(n,2^m)$ to give its output. Now, let $g'(n,2^m)$ be the maximum length of any such program and $A'$ be the algorithm that constructs that program and then pads its length to $g'(n,2^m)$ bits. $A'$ still runs in polynomial time,  the programs it outputs still compute the same function as those $A$ outputs, and its outputs are always $g'(n,2^m)$ bits long and run in exactly $g(n,2^m)$ time.

Now, let $Ob$ be an efficient indistinguishability obfuscation algorithm ,and let $$A''(T,m,(H_1,...,H_{n'}),i,j,h,v)=Ob(A'(T,m,(H_1,...,H_{n'}),i,j,h,v))$$ for all $(T,m,(H_1,...,H_{n'}),i,j,h,v)$. This algorithm still runs in polynomial time, and by correctness of IO it still outputs the code of a program that efficiently computes the appropriate function. Furthermore, given $(T,m,(H_1,...,H_{n'}),i,j,h,v)$ and $(H'_1,...,H'_{n'}), i', j', h', v'$ such that the functions corresponding to $(T,m,(H_1,...,H_{n'}),i,j,h,v)$ and $(T,m,(H'_1,...,H'_{n'}),i',j',h',v')$ are equivalent, the outputs of $A'$ on these inputs have the same length and compute the same result in the same number of timesteps on every input. Therefore, the outputs of $A''$ on these inputs are indistinguishable. So, $A''$ is a censored function obfuscator.
\end{proof}

\subsection{Proof of Theorem~\ref{genFail1}}
With that established, we can finally define the learning algorithm we will be using.

\begin{definition}
Let CFO be a censored function obfuscator as in Definition~\ref{def:cfo}. Then $L$ is the algorithm that takes $((X_1,f(X_1)),...,(X_t,f(X_t)))$ as input and does the following. First, it sets $m= \lceil \log_2(t)\rceil+2$ and returns the constant function $1$ if $m>n'$. Otherwise, it sets $H_k=\{0,1\}^m$ for each $k$, $i=\emptyset$, $h=0^m$, $j=1$. Then, it randomly chooses $v\in\{0,1\}$. Finally, it returns $CFO(\{X_k:1\le k\le t\},m,(H_1,...,H_{n'}),i,j,h,v)$.
\end{definition}

At this point, we can finally prove~Theorem~\ref{genFail1}.

\begin{proof}
Let $D$ and $L$ be as previously defined, $n$ be a positive integer, $(S,f)\sim D(1^n)$, $T$ be a random subset of $S$ of size $g(n)>0$ where $g$ is polynomial in $n$, $T'=\{(x,f(x)): x\in T\}$ and $\widehat{f}=L(T')$. $S[m(\{0,1\}^m,...,\{0,1\}^m)]=S$, so $\widehat{f}$ is an obfuscated version of either $f$ or $f^T$ with both cases being equally likely unless $|T|>|S|/4$. Regardless of which case holds, $\widehat{f}(x)=f(x)$ for all $x\in T$. Also, the fact that $|T|$ is polynomial in $n$ and $|S|=2^{n'}$ implies that $|T|\le |S|/4$ whenever $n$ is sufficiently large. Assume that this holds for the rest of the proof. $f^T(x)=f(x)$ on exactly half of the $x\in S$ if $|S|\ge 2|T|$, so $D$ and $L$ satisfy the first three properties given by the theorem. 

Now, assume that there is an efficient algorithm $A$ such that 
\[\mathbb{P}\left[A\left(\widehat{f},T'\right)=1\middle|\widehat{f}\equiv f\right]\ge 2/3\]
and
\[\mathbb{P}\left[A\left(\widehat{f},T'\right)=1\middle|\widehat{f}\not\equiv f\right]\le 1/3\]
for all sufficiently large $n$. Also, let CFO be the censored function obfuscator used by $L$. We claim that we can break the signature scheme using the following algorithm

\begin{algorithm}
\caption{{\sc SignatureForgingAlgorithm}$(n,PK)$ } 
\label{alg:sigforge}


\begin{enumerate}
\item Let $n'$ be the largest integer such that $\ell(n')\le n$.

\item Randomly draw $B^{(i,j)}\sim\mathbb{F}_2^{i\times n'}$ for each $1\le i,j\le n'$.

\item Randomly draw $T^0\subseteq\{0,1\}^{n'}$ with $|T^0|=g(n)$.

\item Set $T=\{ h_{(SK,PK,B)}(x)\conc 0^{n-\ell(n')}:x\in T^0\}$, using the signing oracle to get the necessary signatures so that the algorithm can compute $h_{(SK,PK,B)}(x)$ without knowing the value of $SK$.

\item Set $m=\lceil\log_2(g(n))\rceil+2$.

\item Set $H_i=\{B^{(m,i)} x^\star| x\in T\}$ for all $0<i\le n'$

\item  Set $H_i^{(r)}=H_i\cup\{h\in\{0,1\}^m: h< r-2^m\cdot(i-1)\}$ for all $i$ and $0\le r\le n' \cdot 2^m$.

\item Set $T'=\{(x,1):x\in T\}$.

\item for each $0\le r\le n'\cdot 2^m$, set 
\[z_r=\sum_{i=1}^{n(n')^2 4^m}  A\left(CFO\left(T,m,H^{(r)},\emptyset,1,0^m,1\right),T'\right)\]
and
\[z'_r=\sum_{i=1}^{n(n')^2 4^m}  A\left(CFO\left(T,m,H^{(r)},\emptyset,1,0^m,0\right),T'\right)\]

\item If there exists $r$ such that $|z_r-z_{r+1}|>n\cdot  n' 2^m/40$ then pick such an $r$ and 
call $\BigGapForge(1)$ 







\item If there exists $r$ such that $|z'_r-z'_{r+1}|>n \cdot  n' 2^m/40$ then pick such an $r$ call 
$\BigGapForge(0)$ 



%


\item Return "Failure"

\end{enumerate}


\end{algorithm}

\begin{algorithm}
\caption{{\sc Big Gap Forge:} $\BigGapForge(b)$} 
\label{alg:biggapforge}

\begin{enumerate}

\item if $H_i^{(r)}=H_i^{(r+1)}$ for all $i$ return ``FAILURE".

\item Let $1\le i\le n'$ be the minimal value for which there exists $h\in\{0,1\}^m$ such that  $h\in H_i^{(r+1)}$ and $h\not\in H_i^{(r)}$.

\item For all $j$, let: 
\[ 
z^{(j)}  := \sum_{k=1}^{n(n')^2 4^m}  A\left(CFO\left(T,m,H^{(r)},i,j,h,b\right),T'\right)
\]

\item For all $j$ 
\begin{enumerate}
    \item if $|z^{(j)}-z'_{r+1}|<|z^{(j)}-z'_r|$ then $s_j \gets 1$
\item else $s_j \gets 0$ 
\end{enumerate} 


\item RETURN $(h\conc 0^{n-m}\conc m\conc i,s)$.
\end{enumerate}

\end{algorithm}


We need to establish a few properties of {\sc SignatureForgingAlgorithm}. First of all, observe that in step 4 this algorithm invokes its signing oracle on every element of $\{B^{(i,j)}x\circ\{0\}^{n-i}\circ i\circ j: x\in T^0, 1\le i,j\le n'\}$ and nothing else. Given these signatures the algorithm can do everything else listed in polynomial time, so it runs efficiently. 

Now, for each $r$ let $p_r$ and $p'_r$ be the probabilities that $A\left(CFO\left(T,m,H^{(r)},\emptyset,1,0^m,1\right),T'\right)$ and $A\left(CFO\left(T,m,H^{(r)},\emptyset,1,0^m,0\right),T'\right)$ return $1$ conditioned on the values of $T$ and $H$ generated by the program. $H^{(n' 2^m)}_i=\{0,1\}^m$ for all $i$, so the probability distributions of $\widehat{f}$ conditioned on $\widehat{f}\equiv f$ and $\widehat{f}\not\equiv f$ are the same as the probability distirubutions of the outputs of $CFO\left(T,m,H^{(n' 2^m)},\emptyset,1,0^m,1\right)$ and $CFO\left(T,m,H^{(n' 2^m)},\emptyset,1,0^m,0\right)$. So, $\mathbb{E}[p_{n'2^m}]\ge 2/3$ and $\mathbb{E}[p'_{n'2^m}]\le 1/3$. That means that 
\[\mathbb{P}\left[p_{n'2^m}-p'_{n'2^m}\ge 1/6\right]\ge 1/6\]
Also, $S[m(H_1,...,H_{n'})]=T$ with probability $1-o(1)$, and if that holds then $CFO\left(T,m,H^{(0)},\emptyset,1,0^m,1\right)$ and $CFO\left(T,m,H^{(0)},\emptyset,1,0^m,0\right)$ return equivalent functions, and thus have indistinguishable output distributions by CFO's obfuscation property. So, 
\[\mathbb{P}\left[|p_0-p'_0|>1/12\right]=o(1)\]
If $p_{n'2^m}-p'_{n'2^m}\ge 1/6$ and $|p_0-p'_0|\le 1/12$ then there must exist $r$ such that either $|p_{r+1}-p_r|\ge 1/(24 n' 2^m)$ or  $|p'_{r+1}-p'_r|\ge 1/(24 n' 2^m)$ by the triangle inequality.

Next, observe that $|z_r-n(n')^2 4^m p_r|\le n\cdot n' 2^m/200$ and $|z'_r-n(n')^2 4^m p'_r|\le n\cdot n' 2^m/200$ for all $r$ with high probability because for fixed values of $T$ and the $H_i$ each of the $z_r$ and $z'_r$ is a sum of $n(n')^2 4^m$ independent identically distributed variables in $[0,1]$ with expectation $p_r$ or $p'_r$. So, if there exists r such that $|p_{r+1}-p_r|\ge 1/(24 n' 2^m)$ or  $|p'_{r+1}-p'_r|\ge 1/(24 n' 2^m)$ then $|z_r-z_{r+1}|>n\cdot  n' 2^m/40$ or $|z'_r-z'_{r+1}|>n\cdot  n' 2^m/40$ respectively with high probability. Also, given any $r$ such that $H_i^{(r)}=H_i^{(r+1)}$ for all $i$, $p_r=p_{r+1}$ and $p'_r=p'_{r+1}$. So, $|z_r-z_{r+1}|\le n\cdot n' 2^m/100$ and $|z'_r-z'_{r+1}|\le n\cdot n' 2^m/100$ for all such $r$ with probability $1-o(1)$. All of this combined implies that the algorithm returns "Failure" with probability at most $5/6+o(1)$.

The algorithm always either returns "Failure" or returns an alleged forged signature, so at this point we just need to show that its alleged forged signatures are usually correct. So, consider the case where the algorithm finds $r$ such that $|z_r-z_{r+1}|>n\cdot  n' 2^m/40$ and there exists $i$ for which $H^{(r+1)}_i\ne H^{(r)}_i$. By the constuction of $H^{(r)}$ and $H^{(r+1)}$, $H^{(r+1)}_{\lfloor r/2^m\rfloor+1} =H^{(r)}_{\lfloor r/2^m\rfloor+1}\cup\{r-2^m\cdot \lfloor r/2^m\rfloor\}$ and $H^{(r+1)}_i=H^{(r)}_i$ for all other $i$. So, the algorithm will set $i=\lfloor r/2^m\rfloor+1$ and $h=r-2^m\cdot \lfloor r/2^m\rfloor$. Now, let $Sig=Sign_{SK}(h\circ 0^{n-m}\circ m\circ i)$. For a given $j$, $S[m(H^{(r)}_1,...,H^{(r)}_{n'})|i,h,j]$ is $S[m(H^{(r+1)}_1,...,H^{(r+1)}_{n'})]$ if the $j$th bit of $Sig$ is $1$ and $S[m(H^{(r)}_1,...,H^{(r)}_{n'})]$ otherwise. So, the probability distribution of the output of $CFO\left(T,m,H^{(r)},i,j,h,1\right)$ is indistinguishable from the distribution of the output of $CFO\left(T,m,H^{(r)},\emptyset,1,0^m,1\right)$ whenever $Sig_j=0$ and indistinguishable from the distribution of the output of $CFO\left(T,m,H^{(r+1)},\emptyset,1,0^m,1\right)$ whenever $Sig_j=1$. So, $|z^{(j)}-z_{r+1}|\le n\cdot n' 2^m/100$ for all $j$ for which $Sig_j=1$ and $|z^{(j)}-z_{r+1}|\le n\cdot n' 2^m/100$ for all $j$ for which $Sig_j=0$ with high probability. If this holds then $s=Sig$ and the algorithm succeeds at returning a pair of an element and its signature. Furthermore, the fact that $h\not\in H_i^{(r)}$ implies that there is no $x\in T^0$ for which $B^{(m,i)} x=h$. That in turn means that the algorithm never needed to invoke the signing oracle on $h\circ \{0\}^{n-m}\circ m\circ i$. So, in this case the algorithm succeeds in forging a signature. The case where the algorithm finds $r$ such that $|z'_r-z'_{r+1}|>n\cdot  n' 2^m/40$ and there exists $i$ for which $H^{(r+1)}_i\ne H^{(r)}_i$ is analogous.

So, {\sc SignatureForgingAlgorithm} is a polynomial time algorithm that succeeds at forging a signature with nonvanishing probability. By the requirement that $(G,Sign, Ver)$ be a unique signature scheme, no such algorithm can exist. Thus, we have a contradiction.
\end{proof}

\section{Unconditional limitations on generalization testing}

\ColinNote{I think the time hierarchy result derived from this line of reasoning would be that for every $c>0$ there exists a function $f\in \mathbf{BPP}$ such that for every randomized algorithm $A$ that runs in $O(n^c)$ time there exists $n$ such that $\mathbb{P}[A(X)=f(X)]\le 1/2+n^{-c}$ when $X$ is drawn uniformly at random from $\{0,1\}^n$.}

Our plan for proving Theorem~\ref{genFail2} is essentially to show that we can define some function that is computable in polynomial time but not in $n^{c'}$ time and define $f$ to compute this function on the first $(c+1/4)\log_2(n)$ bits of its input. Then the learning algorithm can either provide the correct lookup table or provide a table that is correct on the observed samples but not in general, and the generalization testing algorithm will not be able to determine which it did. Our argument that such a function exists will depend on whether one can efficiently compute the permanent of a matrix. The permanent is similar to the determinant but it adds the products corresponding to every permutation rather than subtracting the ones corresponding to odd permutations. More formally, it is defined as follows.

\begin{definition}
Given an $n\times n$ matrix $M$, the permanent of $M$ is given by
\[Perm(M)=\sum_{\sigma\in S_n}\prod_{i=1}^n M_{i,\sigma_i}\]
\end{definition}

A celebrated result in complexity theory is that 
computing the permanent of $\{0,1\}$-valued matrices is $\#P$-hard~\cite{valiant1979complexity}.  
Strictly speaking, the permanent cannot be in $\mathbf{BPP}$ because it is not a decision problem, but we will consider the permanent as being in $\mathbf{BPP}$ if there is an algorithm in $\mathbf{BPP}$ that computes its $i$th digit for all $i$ whenever the entries of the matrix have length polynomial in $n$. That leaves us with two cases.

\subsection{Case: $Permanent\in \mathbf{BPP}$}

If the permanent is efficiently computable, we plan to use it to allow us to determine how likely every sufficiently short algorithm is to return $1$ on each of a polynomial number of possible inputs and then construct the truth table of a function that is essentially uncorrelated with any of them. Our first step is to define a function giving the probability that an algorithm accepts a given input in a given time-frame, which we will do as follows.

\begin{definition}
Given an algorithm $A\in\{0,1\}^*$, $n>0$, $T>0$, and $x\in\{0,1\}^n$, the $T$-step acceptance probability of $A$ on $x$, or $P_T(A,x)$ is the probability that $A$ halts and returns $1$ within $T$ timesteps when it is run on $x$. For the purposes of this definition, we regard our algorithms as being able to draw bits uniformly at random.
\end{definition}

Note that $P_T(A,x)$ is equal to the fraction of random bitstrings in $\{0,1\}^T$ that would cause $A$ to return $1$ when it is run on $x$. Also, the permanent is $\#P$-complete, meaning that given any efficiently computable function, determining how many inputs it returns $1$ on can be efficiently reduced to computing a permanent. So, if $Permanent\in \mathbf{BPP}$ then there is an efficient randomized algorithm that computes $P_T(A,x)$ with high probability for polynomial $T$. For any $c'$, that allows us to construct a function that cannot be approximated in $n^{c'}$ time in the following sense.

\ColinNote{I should probably cite something relating to the fact that permanent is $\#P$-complete and the implications.}

\begin{lemma}
Let $c_1,c_2>0$. If $Permanent\in \mathbf{BPP}$ then there exists a function $g:\{0,1\}^{3c_1\log_2(n)}\times [n^{c_2}] \rightarrow\{0,1\}$ such that $g\in \mathbf{BPP}$ and for every randomized function $g'$ that is uniformly computable in $O(n^{c_1})$ time, 
\[\max_{i\le n^{c_2}} \mathbb{P}[g(X,i)=g'(X,g(X,1),...,g(X,i-1))]\le 1/2+n^{-c_1}\]
for all sufficiently large $n$ when $X$ is drawn uniformly at random from $\{0,1\}^{3c_1\log_2(n)}$.
\end{lemma}

\begin{proof}
Let $A_1,...,A_m$ be a list of all algorithms that can be written in $\ln(\ln(n))$ bits or less. We recursively construct the truth table of $g$ as follows. For a given $x$ and $i$, let $g(x,[i-1])$ denotes the string listing the values of $g(x,j)$ for all $j<i$. Then for each $x$ and $i$ we set $g(x,i)$ to the value in $\{0,1\}$ that minimizes
\[\sum_{j=1}^m \left(\sum_{x'=0}^{x} (g(x',i)-1/2)\left(P_{n^{c_1+1}}(A_j,x'\concat  g(x',[i-1]))-1/2\right)\right)^2\]
or to $0$ if the expression takes on the same value either way. Note that in order to compute $g(x,i)$ we only need to know the values of $g(x',i')$ for $x'\le x$ and $i'\le i$ so we cannot get stuck. Also, the assumption that $Permanent\in \mathbf{BPP}$ implies that we can compute $P_{n^{c_1+1}}$ in $\mathbf{BPP}$ and thus that we can always efficiently compute $P_{n^{c_1+1}}(A_j,x'\concat  g(x',[i-1]))$ with an error probability of $o\left(n^{-3c_1-c_2}\right)$. Furthermore, the computation of $P_{n^{c_1+1}}$ is the only part of this algorithm that is not deterministic, so there is a single function $g$ that we succeed at computing if all of the $P_{n^{c_1+1}}$ are computed accurately. Thus, $g\in \mathbf{BPP}$ as desired. Now, observe that for any $i,x$,
\begin{align*}
&\sum_{j=1}^m \left(\sum_{x'=0}^{x} (g(x',i)-1/2)\left(P_{n^{c_1+1}}(A_j,x'\concat g(x',[i-1]))-1/2\right)\right)^2\\
&\le \frac{1}{2}\sum_{y\in\pm 1/2}\sum_{j=1}^m \left(y\left(P_{n^{c_1+1}}(A_j,x\concat g(x,[i-1]))-1/2\right)+\sum_{x'=0}^{x-1} (g(x',i)-1/2)\left(P_{n^{c_1+1}}(A_j,x'\concat g(x',[i-1]))-1/2\right)\right)^2\\
&\le \frac{m}{16}+\sum_{j=1}^m \left(\sum_{x'=0}^{x-1} (g(x',i)-1/2)\left(P_{n^{c_1+1}}(A_j,x'\concat g(x',[i-1]))-1/2\right)\right)^2
\end{align*}
Given any function $g'$ that is uniformly computable in $O(n^{c_1})$ time, there exists $j$ such that $A_j$ computes $g'$ for all sufficiently large $n$. So, for all sufficiently large $n$ and all $i$ it must be the case that
\begin{align*}
&\mathbb{P}_X[g(X,i)=g'(X,g(X,1),...,g(X,i-1))]\\
&=2^{-3c_1\log_2(n)}\sum_{x\in \{0,1\}^{3c_1\log_2(n)}} \mathbb{P}[g'(x,g(x,1),...,g(x,i-1))=g(x,i)]\\
&=\frac{1}{2}+2^{-3c_1\log_2(n)}\sum_{x\in \{0,1\}^{3c_1\log_2(n)}} 2(g(x',i)-1/2)\left(P_{n^{c_1+1}}(A_j,x'\concat  g(x',[i-1]))-1/2\right)\\
&\le \frac{1}{2}+2^{-3c_1\log_2(n)}\sqrt{\sum_{j'=1}^m\left(\sum_{x\in \{0,1\}^{3c_1\log_2(n)}} 2(g(x',i)-1/2)\left(P_{n^{c_1+1}}(A_{j'},x'\concat  g(x',[i-1]))-1/2\right)\right)^2}\\
&\le \frac{1}{2}+2^{-3c_1\log_2(n)}\sqrt{m2^{3c_1\log_2(n)}/4} \le \frac{1}{2}+2^{-3c_1\log_2(n)/2}\sqrt{\ln(n)} \le \frac{1}{2}+n^{-c_1}\\
\end{align*}
as desired.
\end{proof}

Our plan is to generate a probability distribution by picking a random $k$ and then mapping each $x\in\{0,1\}^n$ to $g(k,x_{[\log_2(n^{c+1})]})$. That way, our learning algorithm should be able to either figure out what $k$ is and return the appropriate truth table in terms of $x_{[\log_2(n^{c+1})]}$ or make something up, and the generalization testing algorithm will not have the computational resources to determine which it did. That allows us to prove the following.

\begin{lemma}
If $Permanent\in \mathbf{BPP}$ then Theorem \ref{genFail2} holds.
\end{lemma}

\begin{proof}
First, let $m=\lfloor \log_2(n^{c+1})\rfloor$ and $g$ be the function stated to exist by the previous lemma for $c_1=\max(c,c')+2$ and $c_2=c+1$. Our algorithm $D$ picks a random $k\in\{0,1\}^{3c_1\log_2(n)}$, sets $S=\{0,1\}^n$, and sets $f$ so that $f(x)=g(k,x_{[m]})$ for all $x$. Our algorithm $L$ draws $V$ uniformly at random from $\{0,1\}$, and then picks a string $s\in \{0,1\}^{2^m}$ as follows. If $V=1$, it picks a random $k'$ such that $g(k',x_{[m]})=y$ for all samples $(x,y)$. Then, it sets $s_i=g(k',i)$ for all $i$. If $V=0$ then it sets $s$ as follows. For each $i$, if the algorithm has a sample $(x,y)$ such that $x_{[m]}=i$ then it sets $s_i=y$; otherwise, it draws $s_i$ uniformly at random from $\{0,1\}$. Either way, it sets $\widehat{f}$ so that $\widehat{f}(x)=s_{x_{[m]}}$ for all $x$ and returns the obvious formulation of it.

$g$ is efficiently computable, there are only polynomially many possible values of $k'$ for $L$ to check, and there are only polynomially many possible values of $x_{[m]}$, so $D$ and $L$ are both efficient algorithms. Also, computing $\widehat{f}(x)$ always comes down to looking up the element of $s$ indexed by the first $O(\log(n))$ bits of its input, which can be done in $O(\log(n))$ time, as desired. Also, $s$ is always defined in such a way that $s_{x_{[m]}}=g(k,x_{[m]})$ for every $x$ contained in the list of samples, so $\widehat{f}$ agrees with $f$ on all of the samples.

Also, given any $k'\ne k$ such that $|\{i\in[2^m]: g(k',i)\ne g(k,i)\}|\ge 2^m/\ln(n)$, the probability that $g(k',x_{[m]})=g(k,x_{[m]})$ for all $x$ in our list of samples is at most $(1-1/\ln(n))^{n^{c}}=o(2^{-3c_1\log_2(n)})$ so with high probability if $V=1$ then $L$ does not pick such a $k'$ and it returns $\widehat{f}$ that agrees with $f$ on $1-o(1)$ of the possible inputs. On the flip side, if $V=0$, then $s_i$ will disagree with $g(k,i)$ on $1/2-o(1)$ of the possible values of $i$ with high probability, so $L$ will return a function that agrees with $f$ on $1/2+o(1)$ of the possible inputs with high probability.

That leaves the task of showing that any algorithm that runs in $O(n^{c'})$ time fails to determine the value of $V$ with $2/3$ accuracy for infinitely many values of $n$. As our first step towards that, consider an arbitrary randomized function $h:\{0,1\}^{2^m}\rightarrow\{0,1\}$ that is computable in $O(n^{c_1})$ uniform time. Then, define $h'$ as the function that appends random bits to its input until it has length $2^m$ and then runs $h$ on it. $h'$ is also computable in $O(n^{c_1})$ uniform time. Now, pick a random $k'$, let $s^0\in\{0,1\}^{2^m}$ be a random string, and let $s^1\in\{0,1\}^{2^m}$ be defined so that $s^1_i=g(k,i)$ for all $i$. Then for all sufficiently large $n$,
\begin{align*}
\mathbb{E}[h(s^1)]-\mathbb{E}[h(s^0)] &=\mathbb{E}[h'(s^1)]-\mathbb{E}[h'(\emptyset)] =\sum_{i=0}^{2^m-1} \mathbb{E}[h'\left(s^1_{[i+1]}\right)]-\mathbb{E}[h'\left(s^1_{[i]}\right)]\\
&=\frac{1}{2}\sum_{i=0}^{2^m-1} \mathbb{E}\left[h'\left(s^1_{[i]}\concat g(k,i+1)\right)\right]-\mathbb{E}\left[h'\left(s^1_{[i]}\concat NOT(g(k,i+1))\right)\right]\\
&=\frac{1}{2}\sum_{i=0}^{2^m-1} (-1)^{g(k,i+1)}\left(\mathbb{E}\left[h'\left(s^1_{[i]}\concat 0\right)\right]-\mathbb{E}\left[h'\left(s^1_{[i]}\concat 1\right)\right]\right)\\
&\le\frac{1}{2}\sum_{i=0}^{2^m-1} 4\cdot n^{-c_1} =O(n^{c+1-c_1})=O(1/n)
\end{align*}

When the sample inputs are generated, they partition $\{0,1\}^{3c_1\log_2(n)}$ into a collection of sets of possible values of $k$ that would result in indistinguishable samples. Then, if $V=1$, $L$ picks a random $k'$ in the same set as $k$. In particular, this means that the probability distribution of $k'$ conditioned on the sample inputs is the uniform probability distribution on $\{0,1\}^{3c_1\log_2(n)}$. That means that given the value of $s$ we can simulate the sample generation process in $O(n^{c+1})$ time.

Now, define a fake sample and function generation algorithm $D'$ that picks a random $s\in\{0,1\}^{2^m}$, sets $\widehat{f}$ so that $\widehat{f}(x)=s_{x_{[m]}}$ for all $x$, selects $X_1,...,X_{n^{c}}$ randomly from $\{0,1\}^n$, and returns $(X_i,\widehat{f}(X_i))$ for all $i$ as samples and the obvious formulation of $\widehat{f}$ as the function. This is equivalent to simulating the sample generation process using a random value of $s$, so the difficulty of distinguishing a string generated by $g$ from a random string implies that no algorithm that runs in $O(n^{c_1}-2^m)=O(n^{c_1})$ uniform time can distinguish a function and set of samples generated by $D'$ from a function and set of samples output by $D$ and $L$ when $V=1$.

That leaves proving that if $V=0$ the results will still be hard to distinguish from the output of $D'$. So, consider having a string $s'$ which is either $s^0$ or $s^1$. Then, randomly draw samples $X_1,...,X_{n^{c}}\in \{0,1\}^n$, and define $s\in\{0,1\}^{2^m}$ such that $s_i=s'_i$ if $i$ is the first $m$ digits of one of the samples and a random bit otherwise. Finally, set $\widehat{f}$ to the function that returns the bit of $s$ indexed by its input, and return $(X_i,\widehat{f}(X_i))$ for all $i$ and $\widehat{f}$. The results of this process have the same probability distribution as the output of $D'$ if $s'$ is a random string, and the same probability distribution as the samples and function output by $D$ and $L$ conditioned on $V=0$ if $s'=s^1$. So, the difficulty of distinguishing between $s^1$ and $s^0$ implies that no algorithm that runs in $O(n^{c_1})$ time can distinguish the samples and function produced by $D'$ from those produced by $D$ and $L$ when $V=0$. So, for any algorithm $A$ that runs in $O(n^{c'})$ time, its probability of returning $1$ when $\widehat{f}\approx f$ and its probability of returning $1$ when $\widehat{f}\not\approx f$ are both within $o(1)$ of its probability of returning $1$ on the function and samples produced by $L'$.
\end{proof}

\subsection{Case: $Permanent not in \mathbf{BPP}$}

If $Permanent\not\in \mathbf{BPP}$, then we plan to find $m<<n$ such that it takes just slightly over $n^c$ time to compute the permanent of an $m\times m$ matrix and then design the setup so that one needs to compute the permanent of $m\times m$ matrices to figure out what $f$ is. That way, an algorithm that runs in $O(n^{c'})$ time should be unable to do it, while slower polynomial time algorithms will be able to. In this case, we will mostly be using permanents mod $p$ for some appropriate prime $p$. Being unable to compute permanents of integer-valued matrices implies being unable to compute permanents of matrices mod most primes because if we can compute the permanent of $m\times m$ matrices mod $p_i$ for $1\le i\le k$ then we can compute the permanent of $m\times m$ matrices mod $p_1p_2...p_k$ by the Chinese Remainder Theorem. That in turn implies that we can compute the permanent of $m\times m$ matrices where all entries are integers with absolute values less than $\sqrt[m]{p_1p_2...p_k/2m!}$.

In order to do this, we will need a fairly efficient algorithm for computing the permanents of small matrices. Our plan for this is to just try every algorithm that can be written in $\log(\log(n))$ bits and then use some properties of the permanent to check if they are actually computing it. More specifically, we will use the following standard properties of the permanent.

\begin{lemma}
Let $M$ be an $m\times m$ matrix. Then
\[Perm(M)=\sum_{i=1}^m M_{1,i}Perm(M_{-1,-i})\]
\end{lemma}

\begin{proof}
This follows from a direct computation:
\begin{align*}
Perm(M)&=\sum_{\sigma\in S_m}\prod_{j=1}^m M_{j,\sigma_j}\\
&=\sum_{i=1}^m\sum_{\sigma\in S_m:\sigma_1=i}\prod_{j=1}^m M_{j,\sigma_j}\\
&=\sum_{i=1}^m M_{1,i} \sum_{\sigma\in S_m:\sigma_1=i}\prod_{j=2}^m M_{j,\sigma_j}\\
&=\sum_{i=1}^m M_{1,i}Perm(M_{-1,-i})
\end{align*}
\end{proof}

\begin{lemma}
Let $M$ and $M'$ be $m\times m$ matrices. Then
\[\sum_{i=0}^{m+1} (-1)^i{ m+1\choose i} Perm(M+i M')=0\]
\end{lemma}

\begin{proof}
The permanent of an $m\times m$ matrix is an $m$th degree polynomial in its entries. So, $Perm(M+i M')$ is an $m$th degree polynomial in $i$. The specified formula is $0$ for any such polynomial, so the lemma holds.
\end{proof}

\ColinNote{Actually, I should probably just cite someone for these rather than proving them.}

That allows us to test whether or not an algorithm claiming to compute the permanent mod $p$ works using the following algorithm.

\begin{algorithm}
\caption{{\sc PermanentComputationTest}}
\label{alg:permtest}

{\bf Input:} Integers $m$ and $n$, prime $p$, and algorithm $A$ that takes an $m\times m$ matrix mod $p$ as input.\\ 

{\bf Output:} A conclusion on whether $A$ computes the permanent reliably.


\begin{enumerate}
\item If $m=1$: \begin{enumerate}
    \item Randomly draw $x_1,...,x_{24n}\sim [p]$.
    \item For each $i$, run $A([[x_i]])$ and check if it returns $x_i$.
    \item If $A([[x_i]])$ returns $x_i$ every time, then return $1$; otherwise, return $0$.
\end{enumerate}

\item Let $A'$ be the algorithm that takes an $m-1 \times m-1$ matrix as input, extends it to an $m\times m$ matrix by giving it a new row and column, setting the entry where they meet equal to $1$, and setting all other new entries to $0$, and running $A$ on the extended matrix.

\item Run {\sc PermanentComputationTest}$(m-1,n,p,A')$ and return $0$ if it returns $0$.

\item For $1\le t\le 6mn:$\begin{enumerate}
    \item Randomly generate $M\in\mathbb{Z}_p^{m\times m}$.
    
    \item If
    \[A(M)\ne \sum_{i=1}^m M_{1,i} A'(M_{-1,-i})\] then return $0$
\end{enumerate}

\item For $1\le t\le 48m^2n:$\begin{enumerate}
    \item Randomly generate $M,M'\in\mathbb{Z}_p^{m\times m}$.
    
    \item If \[\sum_{i=0}^{m+1} (-1)^i{ m+1\choose i} A(M+i M')\ne 0\]
    then return $0$.
\end{enumerate}
\item Return $1$.
\end{enumerate}
\end{algorithm}

This algorithm succeeds at checking whether or not $A$ computes the permanent correctly in the following sense.

\begin{lemma}
Let $m,n$ be positive integers, $p$ be a prime number greater than $m+1$, and $A$ be an algorithm that takes an $m\times m$ matrix as input. When {\sc PermanentComputationTest}$(m,n,p,A)$ is run it runs $A$ $O(m^4n)$ times and otherwise runs efficiently. If $A(M)$ succeeds at computing the permanent of $M$ correctly for each $M\in\mathbb{Z}_p^{m\times m}$ with probability at least $1-1/m^4n^2$ then {\sc PermanentComputationTest}$(m,n,p,A)$ returns $1$ with probability $1-O(1/n)$, while if $A(M)$ miscalculates the permanent of $M$ with probability at least $1/24m^2$ for random $M$ then \\{\sc PermanentComputationTest}$(m,n,p,A)$ returns $0$ with probability at least $1-e^{-n}$.
\end{lemma}

\begin{proof}
The recursion potentially results in {\sc PermanentComputationTest} being run once for each value of $m$ smaller than its original value, and each iteration of the algorithm runs $A$ a maximum of $6mn(m+1)+48m^2n(m+2)$ times. So, in total it runs $A$ a maximum of $O(m^4 n)$ times, and performs efficient computations on its outputs. This algorithm only returns $0$ if the outputs of $A$ fail to satisfy an equation that the permanent satisfies, so if $A$ computes the permanent correctly all $O(m^4 n)$ times it will return $1$. Thus, if $A(M)$ succeeds at computing the permanent of $M$ correctly with probability at least $1-1/m^4n^2$ for all $M$, then {\sc PermanentComputationTest}$(m,n,p,A)$ returns $1$ with probability $1-O(1/n)$.

In order to prove that {\sc PermanentComputationTest}$(m,n,p,A)$ consistently returns $0$ if $A$ fails to compute the permanent correctly, we induct on $m$. If $m=1$ then the algorithm simply checks that $A$ computes the permanent correctly on $24n$ random values. If $A$ has a failure rate of at least $1/24$ then $A$ will be wrong at least once with a probability of at least $1-(23/24)^{24n}\ge 1-e^{-n}$. Now, assume that the desired conclusion holds for $m-1$. Then if $A'$ fails to compute the permanent with accuracy at least $1-1/24(m-1)^2$ the algorithm will return $0$ with probability at least $1-e^{-n}$ due to rejecting $A'$. If $A'$ succeeds at computing the permanent of a random matrix with probability at least $1-1/24(m-1)^2$ but $A$ miscalculates the permanent of a random matrix with probability at least $1/3m$, then every time step $4b$ is run there is at least a $1/3m$ chance that $A(M)$ returns a value other than $Perm(M)$ and at most a $1/6m$ change that $A'(M_{-1,-i})$ returns a value other than $Perm(M_{-1,-i})$ for at least one $i$. That means there is at least a $1/6m$ chance that the algorithm stops and returns $0$ each time, giving it a total probability of returning $0$ of at least $1-(1-1/6m)^{6mn}\ge 1-e^{-n}$. Finally, if the probability that $A$ miscalculates the permanent of a random matrix is less than $1/3m$ but at least $1/24m^2$, then every time step $5b$ is run there is at least a $1/24m^2$ chance that $A(M)$ returns a value other than $Perm(M)$. For each $i\ne 0$, the value of $M+iM'$ is independent of $M$ and equally likely to be any element of $\mathbb{Z}_p^{m\times m}$. So, conditioned on $A(M)$ returning a value other than $Perm(M)$, there is still at least a $1/2$ chance that $A(M+iM')$ returns $Perm(M+iM')$ for each nonzero $i$, and if that happens the sum will be nonzero and the algorithm will return $0$. So, in this case the algorithm still returns $0$ with a probability of at least $1-(1-1/48m^2)^{48m^2n}\ge 1-e^{-n}$. This completes the proof.
\end{proof}

At this point, we could just run {\sc PermanentComputationTest} on all sufficiently short algorithms until we find the one that efficiently computes permanents on the largest matrix. However, a generalization testing algorithm would potentially be able to infer the values of some permanents from the samples, so we need an algorithm that computes permanents with efficiency comparable to the best algorithm for computing permanents of $m\times m$ matrices from the permanents of some randomly selected matrices. As such, we define the following algorithm.

\begin{algorithm}
\caption{{\sc PermanentLearningAlgorithm}}
\label{alg:permlearn}

{\bf Input:} $c>0$, a positive integer $n$, and a prime $p$.

{\bf Output:} An integer $m$ and an algorithm that is intended to compute permanents of $m\times m$ matrices mod $p$.

\begin{enumerate}
\item Let $A_1$,...,$A_{k}$ be a list of all algorithms that can be written in $\ln(\ln(n))$ bits or less.

\item For each $i$, let $A'_i$ be $A_i$ modified so that if it does not terminate within $n^{c+1}$ steps it stops and outputs $0$.

\item Set $m=1$

\item Let $\overline{Perm}_1$ be the algorithm that takes a $1\times 1$ matrix as input and outputs its entry.

\item Repeat the following until one finds an appropriate output: \begin{enumerate}

    \item Increase $m$ by $1$.

    \item Set $\overline{Perm}_m=Null$.

    \item For $1\le t\le n^c$: \begin{enumerate}

        \item Select matrices $M^{(1)},...,M^{({n^c})}$ uniformly at random from $\mathbb{Z}_p^{m\times m}$.

        \item For each $i$, set
        \[p_i=\sum_{j=1}^m M^{(i)}_{1,j}\overline{Perm}_{m-1}\left(M^{(i)}_{-1,-j}\right)\]

        \item For each $1\le i\le k$: \begin{enumerate}
            \item Let $A^\star$ be the algorithm that takes an $m\times m$ matrix as input, inputs $n$, $m$, $p$, that matrix, and $(M^{(1)},p_1),...,(M^{(n^c)},p_{n^c})$ into $A'_i$ and returns the result.
            
            \item Run $PemanentComputationTest(m,n,p,A^\star)$
            
            \item If {\sc PermanentComputationTest} returns $1$, then set $\overline{Perm}_m$ to be the algorithm that takes an $m\times m$ matrix $X$ as input, selects $n$ random $m\times m$ matrices $X'_1$,...,$X'_n$, computes
            \[-\sum_{j=1}^{m+1} (-1)^j{ m+1\choose j} A^\star(X+j X'_{r})\]
            for each $r$, and returns the most common result.
        \end{enumerate}
    \end{enumerate}
   \item If $\overline{Perm}_m$ is still $Null$, then return $m$ and the algorithm that takes an $m\times m$ matrix $X$ as input and returns 
    \[\sum_{j=1}^m X^{(i)}_{1,j}\overline{Perm}_{m-1}\left(X^{(i)}_{-1,-j}\right)\]
    
    \item If $m>\sqrt[5]{n}$ then return $(m, \overline{Perm}_{m})$.
\end{enumerate}
\end{enumerate}
\end{algorithm}

This generally succeeds at finding an $m$ large enought that one cannot compute the permanent of an $m\times m$ matrix in $O(n^c)$ time and learning to compute the permanents of $m\times m$ matrices in the following sense.

\begin{lemma}
Let $c>0$ be a constant, $n$ be a positive integer, and $p$ be a prime of size at most polynomial in $n$. Then {\sc PermanentLearningAlgorithm}$(c,n,p)$ runs in time polynomial in $n$, the algorithm it outputs runs in time polynomial in $n$, the algorithm it outputs computes the permanents of $m\times m$ matrices mod $p$ with accuracy at least $1-n\cdot 3^{-n/2}$ with probability $1-o(2^{-n})$, and for every algorithm $A$, with probability $1-o(2^{-\sqrt{n}})$ {\sc PermanentLearningAlgorithm}$(c,n,p)$ either returns $m>\sqrt[5]{n}$ or returns $m$ such that $A$ cannot compute the permanent of a random $m\times m$ matrix in $n^c\ln(n)$ steps with accuracy at least $1-1/200m$ even given the permanents of a typical set of $n^c$ random matrices.
\end{lemma}

\begin{proof}
First of all, observe that $m$ will never get bigger than $\sqrt[5]{n}+1$ so every loop in this algorithm runs for a number of iterations that is at most polynomial in $n$. Therefore, no step of this algorithm will get executed a superpolynomial number of times. All of the $A'$ run in $O(n^{c+1})$ time, which means that $A^\star$ will also be defined in such a way as to run in polynomial time for every $m$, $t$, and $i$. Every invocation of {\sc PermanentComputationTest} will run in polynomial time because the algorithm it is testing is efficient and it is using polynomially large values of $n$ and $m$. For $m>1$, $\overline{Perm}_m$ makes $n(m+1)$ calls to $A^\star$ and performs an efficient computation on the output. So, {\sc PermanentLearningAlgorithm} runs in time polynomial in $n$, and so does the algorithm it returns.

Next, recall that for any $A^\star$ that fails to compute the permanent with accuracy at least $1-1/24m^2$ then $PemanentComputationTest(m,n,p,A^\star)$ will return $0$ with probability at least $1-e^{-n}$. {\sc PermanentLearningAlgorithm} only runs PemanentComputationTest a number of times polynomial in $n$, so with probability $1-o(2^{-n})$ it never accepts a function that fails to compute the permanent with accuracy at least $1-1/24m^2$. If $A^\star$ does compute the permanent with accuracy at least $1-1/24m^2$ then for any $X$ and a random $X'$,
\[\mathbb{P}\left[Perm(X)=-\sum_{j=1}^{m+1} (-1)^j{ m+1\choose j} A^\star(X+j X'_{r})\right]\ge 1- (m+1)/24m^2\ge 1/12m\]
That means that the probability that at least half of $n$ attempts to use this formula get the wrong answer is less than $2^n (1/12m)^{-n/2}\le 3^{-n/2}$. So, $\overline{Perm}_m$ will compute the permanent with accuracy at least $1-3^{-n/2}$ for every $m$ for which it is not Null with probability $1-o(2^{-n})$. So, the algorithm output by {\sc PermanentLearningAlgorithm} will have an accuracy of at least $1-n 3^{-n/2}$ with probability $1-o(2^{-n})$.

Now, let $A$ be an algorithm that attempts to compute the permanent of an $m\times m$ matrix mod $p$ given $n$, $m$, $p$, and $n^c$ random $m\times m$ matrices with their permanents. Next, let $A'$ be the modified form of $A$ that stops and returns $0$ if it has not terminated within $n^c\ln^2(n)$ steps. Finally, let $\overline{A}$ be the algorithm that attempts to compute the permanent of a matrix $X$ by picking $\sqrt[5]{n}$ random matrices $X_1',...,X'_{\sqrt[5]{n}}$, attempting to compute the value of $Perm(X+iX'_j)$ using $A'$ for each $i$ and $j$, and then attempting to compute the permanent of $X$ by using the fact that 
\[\sum_{i=0}^{m+1} (-1)^i{ m+1\choose i} Perm(X+i X'_j)=0\]
for all $j$ and assuming as few of the $Perm(X+i X'_j)$ as possible were miscalculated by $A'$. If $A$ computes the permanent of a random $m\times m$ matrix correctly in $O(n^c\log(n))$ time with probability at least $1-1/100m$ given the samples it has then $\overline{A}$ compute the permanent of every $m\times m$ matrix with accuracy at least $1-2^{\sqrt[5]{n}} 50^{-\sqrt[5]{n}/2}\ge 1-3^{-\sqrt[5]{n}}$ given the same set of sample permanents. Also, $\overline{A}$ always runs in $O(m\sqrt[5]{n} n^c\ln^2(n)+m^3\sqrt[5]{n}\log(n))=o(n^{c+1})$ time. So, for all sufficiently large $n$, one of the $A'_i$ will compute $\overline{A}$. That means that with probability $1-o(2^{-\sqrt{n}})$, for every $m$ the algorithm tries for which $A$ can compute the permanent of a random $m\times m$ matrix in $O(n^c\log(n))$ time with accuracy at least $1-1/200m$ given $n^c$ random matrices and their permanents with probability at least $1/\ln(n)$ the algorithm will succeed at finding $\overline{Perm}_m$ that computes the permanent of $m\times m$ matrices. So, with probability $1-o(2^{-\sqrt{n}}))$ either the algorithm will return $m>\sqrt[5]{n}$ or the algorithm will return $m$ for which $A$ cannot compute the permanent of an $m\times m$ matrix with this accuracy given the permanents of random matrices.
\end{proof}

While this is a start, we want the generalization testing algorithm to have accuracy much lower than $1-1/200m$ at predicting the value of $f$ on new inputs. Fortunately, \cite{impagliazzo2010uniform} contains the following hardness amplification result.

\begin{lemma}\cite{impagliazzo2010uniform}
For every $\epsilon>0$ and $k>0$, there exists $\delta=O((\log(1/\epsilon))/k)$ and an efficient algorithm $A$ such that the following holds. Let $g:\{0,1\}^n\rightarrow \{0,1\}$ be a function and $g':\{0,1\}^{kn}\rightarrow \{0,1\}$ be defined so that $g'(x_1,...,x_k)=XOR(g(x_1),...,g(x_k))$. Given any function $h:\{0,1\}^{kn}\rightarrow\{0,1\}$ that agrees with $g'$ on at least $1/2+\epsilon$ of the possible inputs, $A$ uses a query oracle to $h$ and produces a list of algorithms $h'_1,...,h'_{1/\epsilon^2}$ such that with probability $\Omega(1)$ there exists $i$ such that $h'_i(x)$ computes $g(x)$ with accuracy at least $1-\delta$ given oracle access to $h$.
\end{lemma}

\ColinNote{I am not sure I have all the fine print right here.}

This suggests the following idea for a function that one will be able to compute less accurately in $O(n^{c'})$ time.

\begin{definition}
For arbitrary positive integers $m,k,p$, let $xPerm_{m,k,p}:\mathbb{F}_p^{m\times m\times k}\times [\lceil \log_2(p)\rceil ]^k\rightarrow\{0,1\}$ be defined such that given $M_1,...,M_k\in\mathbb{F}_p^{m\times m}$ and $i_1,...,i_k\in [\lceil \log_2(p)\rceil ]$, \\$xPerm_{m,k,p}(M_1,i_1,M_2,i_2,...,M_k,i_k)$ is the XOR of the $i_j$th bit of $Perm(M_j)$ over all $j$.
\end{definition}

This is difficult to compute given appropriate parameters in the following sense. 
\begin{lemma}\label{xPermHard}
Let $c_1>0$ be a constant, and $p$ be a prime of size at most polynomial in $n$. Then there exists a constant $c_2>0$ such that for every algorithm $A$ with probability $1-o(2^{-\sqrt{n}})$ either {\sc PermanentLearningAlgorithm}$(c_2,n,p)$ returns $m>\sqrt[5]{n}$ or it returns $m$ such that $A$ cannot compute the value of $xPerm_{m,\sqrt[4]{n},p}$ with accuracy $1/2+n^{-c_1}$ in $O(n^{c_1})$ timesteps even given its value on $n^{c_1}$ random inputs.
\end{lemma}

\begin{proof}
Let $\epsilon=n^{-c_1}/2$, $k=\sqrt[4]{n}$, and $0<m\le \sqrt[5]{n}+1$. By the previous lemma, there exists $\delta=O(\log(n)/\sqrt[4]{n})$ and an efficient algorithm that converts any function that agrees with $xPerm_{m,\sqrt[4]{n},p}$ on at least $1/2+n^{-c_1}$ of the possible inputs into a list of $4n^{2c_1}$ functions such that with probability $\Omega(1)$ at least one of them computes the $i$th digit of the permanent of a random $m\times m$ matrix mod $p$ correctly for at least $1-\delta$ of the possible choices of the matrices and $i$. A little more precisely, there must exist a constant $c_3$ such that the conversion algorithm and the functions it returns both make $O(n^{c_3})$ queries to the function in question and otherwise run in $O(n^{c_3})$ time.

Now, let $c_2=4c_1+c_3+2$, set $m$ to the value given by {\sc PermanentLearningAlgorithm}$(c_2,n,p)$, and let $A$ be an algorithm that attempts to compute the value of $xPerm_{m,\sqrt[4]{n},p}$ given its value on $n^{c_1}$ random inputs. Next, let $A'$ be the algorithm that runs $A$ for up to $n^{c_1}\ln(n)$ timesteps and returns $0$ if it has not terminated. Now, consider a value of $n$ and set of $n^{c_1}$ samples for which $A'$ succeeds at computing $xPerm_{m,\sqrt[4]{n},p}$ with accuracy at least $1/2+n^{-c_1}/2$. Next, let $A^\star_1,...A^\star_{n^{2c_1}\ln(n)}$ be the functions generated by making $\ln(n)$ attempts at converting $A'$ to a function that computes the $i$th bit of the permanent of a random $m\times m$ matrix mod $p$. By the previous lemma, with high probability at least one of these function will compute the $i$th digit of the permanent of a random $m\times m$ matrix with accuracy at least $1-\delta$. So, choose $\sqrt[5]{n}\ln^3(n)$ random pairs of an $m\times m$ matrix $M_t$ and an index $i_t\in[\lceil \log_2 p\rceil]$, and check whether any of the $A^\star_j$ compute the $i_t$th digit of $Perm(M_t)$ correctly for all $t$. Then, set $\overline{A}$ to the first $A^\star_j$ that got them all right if any, or the $0$ function if none did. With high probability every one of these $A^\star$ that fails to compute the $i$th digit of the permanent with accuracy at least $1-1/400m\log_2(2p)$ will get at least one of these samples wrong and the first $A^\star$ that computes it with accuracy at least $1-\delta$ will get them all right. So, with high probability $\overline{A}$ will compute the $i$th digit of the permanent with accuracy at least $1-1/400m\log_2(2p)$. That means that concatenating its conclusions for all the bits of the permanent will allow one to compute the permanent with accuracy at least $1-1/400m$. Furthermore, this whole process runs in $O(n^{3c_1+c_3+1})$ time, so given $O(n^{c_2})$ time we can try it $O(n^{c_1}\ln(n))$ times with different sets of random inputs that we know the values of $Perm$ or $xPerm$ on. So, if the expected accuracy of $A'$ given the value of $xPerm_{m,\sqrt[4]{n},p}$ on a random set of $n^{c_1}$ values is at least $1/2+n^{-c_1}$ for a specific value of $(n,m)$, this procedure gives us a way to compute the permanent of a random $m\times m$ matrix with expected accuracy better than $1-1/200m$ in $O(n^{c_2})$ time for that value of $(n,m)$. We know that either $m>\sqrt[5]{n}$ or the algorithm fails to do that for most choices of sample permanents with probability $1-o(2^{-\sqrt{n}})$. Therefore, with probability $1-o(2^{-\sqrt{n}})$ either {\sc PermanentLearningAlgorithm}$(c_2,n,p)$ returns $m>\sqrt[5]{n}$ or it returns $m$ such that $A$ cannot compute the value of $xPerm_{m,\sqrt[4]{n},p}$ with accuracy $1/2+n^{-c_1}$ in $O(n^{c_1})$ timesteps even given its value on a typical set of $n^{c_1}$ random inputs, as desired.
\end{proof}

So, that gives us a function that cannot be computed in $O(n^{c_1})$ time with accuracy nontrivially better than that attained by guessing blindly unless {\sc PermanentLearningAlgorithm}$(c_2,n,p)$ returns $m>\sqrt[5]{n}$. Our next order of business is to argue that if Permanent $\not\in \mathbf{BPP}$ then if we run {\sc PermanentLearningAlgorithm}$(c,n,p)$ it will usually not return $m>\sqrt[5]{n}$. More formally, we claim the following.

\begin{lemma}
If $Permanent\not\in \mathbf{BPP}$ then for any constant $c>0$ there are infinitely many pairs of a positive integer $n$ and a prime $p\le n$ for which {\sc PermanentLearningAlgorithm}$(c,n,p)$ returns $m\le\sqrt[5]{n}$ with high probability.
\end{lemma}

\begin{proof}
Assume on the contrary that there is a $c$ such that {\sc PermanentLearningAlgorithm}$(c,n,p)$ returns $m>\sqrt[5]{n}$ with nonvanishing probability for all sufficiently large $n$ and prime $p\le n$. {\sc PermanentLearningAlgorithm}$(c,n,p)$ almost always provides an algorithm for computing the permanent of an $m\times m$ matrix mod $p$, and for constant $c$ and $p\le n$ both {\sc PermanentLearningAlgorithm} and the algorithm it produces run in time polynomial in $n$. Now, let $m=\lfloor \sqrt[5]{n}\rfloor+1$. For sufficiently large $n$ we can run {\sc PermanentLearningAlgorithm}$(c,n,p)$ for each prime less than or equal to $n$ and repeat it $\ln^2(n)$ times in order to get an algorithms that compute the permanent of $m\times m$ matrices mod $p$ for every such $p$ with high probability. That in turn gives us an efficient algorithm that computes such permanents mod the product of all primes less than $n$. By the Prime Number Theorem, there are $\Omega(n/\log(n))$ primes less than $n$, so there must exist $\epsilon>0$ such that their product is at least $e^{\epsilon n}$ for all sufficiently large $n$. Any $m\times m$ matrix in which all elements have absolute value at most $2^m$ will have permanent at most $m!\cdot 2^{m^2}\le(\sqrt[5]{n}+1)!2^{(\sqrt[5]{n}+1)^2}=o(2^{\epsilon n})$, so we can efficiently compute the permanent of any such matrix by finding the integer of smallest absolute value that has the correct residue mod $p$ for every $p\le n$. That in turn lets us efficiently compute the permanent of an arbitrary $m\times m$ matrix mod $p$ for any $p< 2^m$. Finally, that allows us to efficiently compute the permanent of any $m\times m$ matrix with entries of length polynomial in $n$ by computing its permanent mod $p$ for increasingly large $p$ until their product is at least twice the $m$th power of its largest entry times $m!$ and then concluding that its permanent is the smallest integer that has the correct residue mod all of the primes in question. So, $Permanent\in \mathbf{BPP}$. Contradiction.
\end{proof}

At this point, we are finally ready to explain the distribution and learning algorithms we will be using in this case. For our distribution, we plan to create a list of values of $xPerm$ on random lists of matrices and digits, include the matrices and digits in the inputs generated by our distribution, and force one to compute the xPerm in question in order to determine the correct output. More formally, we use the following algorithm.

\ColinNote{This is not the right formatting. I need to fix it after I upload it.}

\begin{algorithm}
\caption{D}
\label{alg:D2} 

Input: $c,c'>0$ and a positive integer $n$.

Output: An efficiently samplable set $S\subseteq\{0,1\}^n$ and a function $f:S \rightarrow\{0,1\}$

\begin{enumerate}
\item Set $c''$ to the value corresponding to $\max(c,c')+2$ as given by lemma \ref{xPermHard}.

\item For each prime $p\le n$, set $(m_p,A_p)=${\sc PermanentLearningAlgorithm}$(c'',n,p)$.

\item Choose $p$ that minimizes the value of $m_p$, and set $m=m_p$ and $A=A_p$. 

\item Set $l=\lfloor (c+1/4)\log_2(n)\rfloor$, and $k=\sqrt[4]{n}$

\item For each $x\in\{0,1\}^l$, randomly draw $M^{(x,1)},...,M^{(x,k)}\sim \mathbb{F}_p^{m\times m}$ and $i^{(x,1)},...,i^{(x,k)}\sim \left[\lceil \log_2(p)\rceil\right]$.

\item For each $x\in\{0,1\}^l$, set $y_{x}$ to the XOR over all $j$ of the $i_j$th digit of $A(M^{(x,j)})$.

\item Let $S'$ be the set of all strings of the form $x||M^{(x,1)}||i^{(x,1)}||M^{(x,2)}||i^{(x,2)}||...M^{(x,k)}||i^{(x,k)}$, and $r$ be the length of these string.

\item Let $S$ be the set of all strings formed by picking $x\in\{0,1\}^l$, starting with $m||p||x$, concatenating random elements of $S'$ until the string's length exceeds $n-r$, and then adding $0$s at the end to bring its length up to $n$.

\item Let $f$ be the function that maps every string starting with $m||p||x$ to $y_x$ for all $x\in\{0,1\}^l$.

\end{enumerate}
\end{algorithm}

Our learning algorithm will either compute all of the values of xPerm in question and output an appropriate lookup table, or output a lookup table that has values corresponding to elements of $x$ that occurred in the sample set correctly and all other values set randomly. More formally, it will do the following.

\begin{algorithm}
\caption{L}
\label{alg:l2}

Input: $c,c'>0$, a positive integer $n$, and samples $(X_1,y_1),...,(X_{n^{c}},y_{n^{c}})$.

Output: A function from $\{0,1\}^n$ to $\{0,1\}$.

\begin{enumerate}
\item Set $c''$ to the value corresponding to $\max(c,c')+2$ as given by lemma \ref{xPermHard}.

\item Read off $m$ and $p$ from the beginning of $X_1$.

\item Set $l=\lfloor (c+1/4)\log_2(n)\rfloor$, and $k=\sqrt[4]{n}$.

\item Set $r$ equal to the length of a string of the form $x||M^{(x,1)}||i^{(x,1)}||M^{(x,2)}||i^{(x,2)}||...M^{(x,k)}||i^{(x,k)}$ for $x\in\{0,1\}^l$, $M^{(x,j)}\in\mathbb{F}_p^{m\times m}$ for all $x,i$ and $i^{(x,j)}\in[\lceil\log_2(p)\rceil]$ for all $x,i$.

\item Set $T'$ equal to the list of all $x\in\{0,1\}^l$ such that $X_t$ starts with $m||p||x$ for some $t$.

\item For every $X_t$, divide the part of $X_t$ after the initial substring of $m||p||x$ into blocks of length $r$, and let $R$ be the set of all such blocks contained in any of the $X_t$.

\item Set $t=0$, $A=Null$.

\item While $t<n^2$ and $A=Null$: \begin{enumerate}

\item Set $(m',A')=${\sc PermanentLearningAlgorithm}$(c'',n,p)$

\item If $m'=m$ then set $A=A'$

\item Increase $t$ by 1
\end{enumerate}

\item For each $x\in\{0,1\}^l$, if $R$ has an element starting with $x$ then define $M^{(x,j)}\in\mathbb{F}_p^{m\times m}$ and $i^{(x,j)}\in[\lceil\log_2(p)\rceil]$ so that the element in question is $x||M^{(x,1)}||i^{(x,1)}||M^{(x,2)}||i^{(x,2)}||...M^{(x,k)}||i^{(x,k)}$.

\item For each $x\in\{0,1\}^l$, set  $y_{x}$ to the XOR over all $j$ of the $i_j$th digit of $A(M^{(x,j)})$, or to a random element of $\{0,1\}$ if those have not been defined.

\item Randomly choose $v\in\{0,1\}$.

\item If $v=1$, then set $s=y$.

\item If $v=0$ then set $s_x=y_x$ for all $x\in T'$ and draw all other elements of $s$ randomly from $\{0,1\}$.

\item Return the function that finds $x\in\{0,1\}^l$ such that its input starts with $m||p||x$ and returns $s_x$.

\end{enumerate}
\end{algorithm}

At this point, we can finally prove that if the permanent is not efficiently computable then the theorem holds.

\begin{lemma}
If $Permanent\not\in BPP$ then theorem \ref{genFail2} holds.
\end{lemma}

\begin{proof}
We use the distribution and learning algorithms defined above, and our first order of business is to prove that they run efficiently when $D$ is used to generate $n^{c}$ samples and $L$ is run on them. Step $1$ in $D$ is easy. Step two needs to determine which of the integers less than $n+1$ are prime, and then run PermanentLearningAlgorithm on them with $c,n$, both of which can be done in polynomial time for a fixed value of $c$. Choosing $m$, $A$, $l$, and $k$ is easy, and there are only a polynomial number of elements in $\{0,1\}^l$ so we can select the $M^{(x,j)}$ and $i^{(x,k)}$ in polynomial time. Computing all of the $y_x$ requires running $A$ a polynomial number of times, and is thus also efficient, and $S'$ has a polynomial number of efficiently computable elements. The stated method of sampling from $S$ is efficient, and so is the method for computing $f$, so $D$ runs efficiently. In $L$, the first $7$ steps are clearly efficient, while step $8$ requires running PermanentLearningAlgorithm$(c,n,p)$ a maximum of $n^2$ times. Step $9$ requires performing an efficient computation on every element of a polynomial sized set, step $10$ requires running $A$ a polynomial number of times, and steps 11-14 are clearly efficient. So, $L$ runs efficiently. 

Next, we claim that with high probability $L$ succeeds in retrieving the values of $M$ and $i$ generated by $D$ and they both set $y_x=xPerm_{m,k,p}(M^{(x,1)}||i^{(x,1)}||...||M^{(x,k)}||i^{(x,k)})$ for all $x$. First, observe that $r=O(k m^2\log(p))=O(n^{13/20}\log(n))$, which means that each of the samples produced using $D$ contains $\Omega(n^{7/20}/\log(n))$ blocks of the form $x||M^{(x,1)}||i^{(x,1)}||M^{(x,2)}||i^{(x,2)}||...M^{(x,k)}||i^{(x,k)}$. That means that there are $\Omega(n^{c+7/20}/\log(n))$ such blocks in the full set of samples, and there are $2^l=O(n^{c+1/4})$ possible values of $x$. So, with probability $1-o(1)$ every value of $x$ has a block listing the corresponding values of $M$ and $i$ in the samples, in which case $L$ succeeds at extracting them all. When $D$ is run, $y_x$ gets set to $xPerm(M^{(x,1)}||i^{(x,1)}||...||M^{(x,k)}||i^{(x,k)})$ for all $x$ unless $A$ miscalculates a permanent at least once. The algorithm only runs $A$ a polynomial number of times and it has an error rate of at most $n\cdot 3^{-n/2}$ with high probability, so $D$ sets all of the $y$ correctly with high probability. Also, there are $O(\sqrt[5]{n})$ possible values of $m$ that PermanentLearningAlgorithm$(c'',n,p)$ could pick for each value of $p$ and $D$ runs PermanentLearningAlgorithm at most $n$ times in the process of picking $(m,p)$. So, with high probability $m_p$ takes on a value that PermanentLearningAlgorithm$(c'',n,p)$ outputs with probability at least $n^{-7/5}$ for every $p$. So, $L$ gets the same value of $m$ in some iteration of step 7a with high probability, which leaves it with an $A$ that also computes the permanents of $m\times m$ matrices with accuracy at least $1-n\cdot 3^{-n/2}$ with high probability. That means that $L$ also sets $y_x=xPerm(M^{(x,1)}||i^{(x,1)}||...||M^{(x,k)}||i^{(x,k)})$ for all $x$ with high probability, as desired. In particular, this means that $D$ and $L$ compute the same value of $y$ with high probability. 

If that holds and $L$ sets $v=1$ then $\widehat{f}=f$. If this holds and $L$ sets $v=0$ then $\widehat{f}$ will agree with $f$ on all inputs that start with a string of the form $m||p||x$ that is the start of one of the sample inputs and take on random values for other values of $x$. There are $n^c$ sample inputs and $\Omega(n^{c+1/4})$ possible values of $x\in \{0,1\}^l$, so in this case $\widehat{f}$ will agree with $f$ on $1/2+o(1)$ of the possible inputs with high probability. 

That leaves the task of proving that no algorithm running in $O(n^{c'})$ time can determine whether $v$ was $0$ or $1$ with $2/3$ accuracy. So, assume that there exists an algorithm $A$ that runs in $O(n^{c'})$ time which succeeds at determining which case holds with at least $2/3$ accuracy for all sufficiently large $n$. Then we propose the following algorithm for computing $xPerm_{m,k,p}(M_1||i_1||...||M_k||i_k)$ for random $M_1,i_1,...,M_k,i_k$ given matrices with known permanents $M^{\star(x,i)}$ for each $x\in\{0,1\}^l$ and $1\le i\le k$.

\ColinNote{$m$ and $p$ are parameters of this.}

\begin{algorithm}
\caption{{\sc PermanentReduction}}
\label{alg:permreduction}


\begin{enumerate}
\item Set $k=\sqrt[4]{n}$ and $l=\lfloor(c+1/4)\log_2(n)\rfloor$.

\item Randomly select $0 \le t< 2^l$.

\item Randomly draw $\overline{X}_1,...,\overline{X}_{n^c}\sim\{0,1\}^l$

\item Set $T'=\{\overline{X}_1,...,\overline{X}_{n^c}\}$.

\item For each $x\in\{0,1\}^l$ and $1\le j\le k$, set
\[M^{(x,j)}= \begin{cases}
M^{\star(x,j)} &\text{ if } x\in T'\text{ or } x<t \\
M_j &\text{ if } x=t\not\in T'\\
\text{A random sample from }\mathbb{F}_p^{m\times m} &\text{ otherwise}
\end{cases}\]

\item For each $1\le j\le k$, set $i^{(t,j)}=i_j$.

\item Draw $i^{(x,j)}$ randomly from $[\lceil \log_2(p)\rceil ]$ for all other $x$ and $j$.

\item Let $S'$ be the set of all strings of the form $x||M^{(x,1)}||i^{(x,1)}||M^{(x,2)}||i^{(x,2)}||...M^{(x,k)}||i^{(x,k)}$, and $r$ be the length of these strings.

\item For each $1\le j\le n^c$, assign a value to $X_j$ by starting with $m||p||\overline{X}_j$, concatenating random elements of $S'$ until its length exceeds $n-r$, and then concatenating $0$s to bring its length up to $n$.

\item Let $s'\in\{0,1\}^{2^l}$ be a random string.

\item For each $x\in\{0,1\}^l$, set $s_x=xPerm_{m,k,p}(M^{(x,1)}||i^{(x,1)}||...||M^{(x,k)}||i^{(x,k)})$ if $x\in T'$ or $x<t$ and $s'_x$ otherwise.

\item Let $\widehat{f}$ be the function that finds $x\in\{0,1\}^l$ such that its input starts with $m||p||x$ and returns $s_x$.

\item Run $A((X_1,\widehat{f}(X_1)),...,(X_{n^c},\widehat{f}(X_{n^c})),\widehat{f})$.

\item If $A$ says that $\widehat{f}$ has high accuracy, then conclude that $xPerm_{m,k,p}(M_1||i_1||...||M_k||i_k)=s'_t$; otherwise, conclude that $xPerm_{m,k,p}(M_1||i_1||...||M_k||i_k)=NOT(s'_t)$.
\end{enumerate}

\end{algorithm}

This procedure has a few key properties. First of all, the algorithm only uses matrices that we already knew the permanents of to compute xPerm's in step 11, so it never actually needs to compute the permanent of a matrix. As a result, this whole algorithm runs in $O(n^{c'}+n^{c+2})$ time. Secondly, the probability distribution of the inputs to $A$ conditioned on $t=2^l$ would be within $o(1)$ of the probability distribution of the samples output by $D$ and function output by $L$ conditioned on $v=1$, while the probability distribution on the inputs to $A$ conditioned on $t=0$ is within $o(1)$ of the probability distribution of the samples output by $D$ and function output by $L$ conditioned on $v=1$. Thirdly, for any $t_0>0$, the probability distribution of the inputs to $A$ conditioned on $t=t_0$ are the same as the probability distribution of the inputs to $A$ conditioned on $t=t_0-1$ and $s'_{t}=xPerm_{m,k,p}(M_1||i_1||...||M_k||i_k)$. So, if we let $Z$ be $1$ if $A$ concludes that $\widehat{f}$ has high accuracy and $0$ otherwise  and $W=xPerm_{m,k,p}(M_1||i_1||...||M_k||i_k)$ then the probability that this algorithm computes $W$ correctly is:

\begin{align*}
&P[W+Z+s'_t\equiv 1\pmod{2}]\\
&=\frac{1}{2^l}\sum_{t_0=0}^{2^l-1} P[W+Z+s'_t\equiv 1\pmod{2}|t=t_0]\\
&=\frac{1}{2^l}\sum_{t_0=0}^{2^l-1} P[Z=1,W=s'_t|t=t_0]+P[Z=0,W\ne s'_t|t=t_0]\\
&=\frac{1}{2^l}\sum_{t_0=0}^{2^l-1} P[Z=1,W=s'_t|t=t_0]+P[Z=0|t=t_0]-P[Z=0,W= s'_t|t=t_0]\\
&=\frac{1}{2^{l+1}}\sum_{t_0=0}^{2^l-1} 2P[Z=0|t=t_0]+P[Z=1|W=s'_t,t=t_0]-P[Z=0|W= s'_t,t=t_0]\\
&=\frac{1}{2^{l+1}}\sum_{t_0=0}^{2^l-1} 2P[Z=0|t=t_0]+P[Z=1|t=t_0+1]-P[Z=0|t=t_0+1]\\
&=\frac{1}{2^{l+1}}\sum_{t_0=0}^{2^l-1} 2P[Z=0|t=t_0]+1-2P[Z=0|t=t_0+1]\\
&=1/2+\frac{1}{2^l}\sum_{t_0=0}^{2^l-1} P[Z=0|t=t_0]-P[Z=0|t=t_0+1]\\
&=1/2+\frac{P[Z=0|t=0]-P[Z=0|t=2^l]}{2^l}\\
&\ge 1/2+\frac{1-o(1)}{3\cdot 2^l}\\
&=1/2+\Omega(n^{-c'})
\end{align*}

However, we know that any algorithm running this quickly must fail to compute the value of $W=xPerm_{m,k,p}(M_1||i_1||...||M_k||i_k)$ with this high an accuracy for infinitely many values of $n$ if $Permanent\not\in BPP$ so we have a contradiction. Thus, no such algorithm $A$ can exist, which completes the proof.

\ColinNote{This probably needs editing.}
\end{proof}

\bibliographystyle{alpha}
\bibliography{biblio}

\end{document}